\documentclass[a4paper]{article}
\usepackage[affil-it]{authblk}

\newcommand{\firstpage}{
\renewcommand{\Authfont}{\normalsize}

\title{\large\bfseries Regular expressions for decoding of neural network outputs}
\date{\small \today}
\author[1]{Tobias Strau\ss}
\author[1]{Gundram Leifert}
\author[1]{Tobias Gr\"uning}
\author[1]{Roger Labahn}
\affil[1]{Department of Mathematics, University of Rostock, Rostock, Germany}
\maketitle
\begin{abstract}
This article proposes a convenient tool for decoding the output of neural networks trained by Connectionist Temporal Classification (CTC) for handwritten text recognition. We use regular expressions to describe the complex structures expected in the writing. The corresponding finite automata are employed to build a decoder. We analyze theoretically which calculations are relevant and which can be avoided. A great speed-up results from an approximation. {We conclude that the approximation most likely fails if the regular expression does not match the ground truth which is not harmful for many applications since the low probability will be even underestimated}. The proposed decoder is very efficient compared to other decoding methods. The variety of applications reaches from information retrieval to full text recognition. We refer to applications where we integrated the proposed decoder successfully.
\end{abstract}
}

\usepackage[T1]{fontenc}
\usepackage[utf8]{inputenc}
\usepackage{amssymb}
\usepackage{amsmath, amsthm}
\usepackage{graphicx}
\usepackage{tikz}
\usepackage{bbm}
\usepackage[ruled,vlined]{algorithm2e}
\usepackage[plainpages=false,pdfpagelabels]{hyperref}
\usepackage{subfigure}
\usepackage{caption}
\usepackage{pifont}
\usepackage{gnuplottex}
\usepackage{gnuplot-lua-tikz}

\usetikzlibrary{automata,positioning}

\newcommand{\IN}{\mathbb{N}}

\newcommand{\raum}[1]{\mathbb{#1}}
\newcommand{\fett}[1]{\pmb{ {#1}}}
\newcommand{\albet}{\Sigma}
\newcommand{\dict}{\mathcal{V}}

\newcommand{\nac}{\emph{NaC}}
\newcommand{\symbnac}{\star} 
\newcommand{\likeliest}{\gamma}
\newcommand{\emission}{\zeta}

\DeclareMathOperator*{\argmax}{arg max}

\DeclareMathOperator{\app}{app}
\DeclareMathOperator{\cont}{cont}
\DeclareMathOperator{\p}{P}

\theoremstyle{definition}
\newtheorem{theo}{Theorem}

\newtheorem{ex}[theo]{Example}

\newtheorem{defi}[theo]{Definition}
\newtheorem{rem}[theo]{Remark}

\begin{document}

\firstpage

\section{Introduction}
{Sequence labeling is the task of assigning a (class) label to each position of an incoming sequence such as speech or handwriting recognition. These tasks are typically very complex and even subproblems are challenging. This article focuses on the decoding problem i.e. finding the most likely label sequence for a given output of a classifier such as neural networks (NNs), Hidden Markov Models (HMMs) or Conditional Random Fields (CRFs).
}

Deep learning methods has pushed the research of complex tasks such as handwritten text recognition (see \cite{graves2009offline}). The special needs of such complex tasks require advanced decoding methods. For example, a typical subproblem in full text recognition is structuring the recognizers output into a sequence of regions of words, punctuations and numbers.
In many cases, the most likely label sequence yields an acceptable segmentation. However, it happens that this label sequence is not feasible i.e. it does not match the expected structure and has to be corrected. Finding the optimal feasible structure is one of many applications of this article. For this aim, we describe feasible structures by regular expressions -- a powerful pattern sequence which is used in nearly all computational text processing systems such as text editors and programming languages like Java or Python. We then derive an algorithm based on finite automata that yields the most likely label sequence fitting the previously described regular expression.

Beyond finding the optimal feasible label sequence fitting an expected structure (regular expression), we gain several other features since we also consider the functionality of capturing groups.
A capturing group defines a part of the regular expression. The associated part of the matching label sequence can be used to structure the decoding result for further analysis. In case of our previous example, we obtain a complete segmentation into words, numbers and symbols without additional parsing facilitating the calculation of the matching subsequence and the likelihood. We just define word, number and symbol capturing groups. The complete decoding can be done in a few lines of code.

Keyword spotting is another obvious application which can be solved very conveniently.
The keyword is either the beginning of the line or there is a space or another separating symbol (quotation marks, opening parenthesis, etc.) before the keyword. With the common notation of regular expressions, this pattern may be captured by inserting {\tt (.*(?<pre>[ "(-]))?} before the keyword which means: If there is anything before the keyword, it ends with at least one of the aforementioned symbols. This last symbol (if there is one) is contained in the capturing group {\tt pre}. Information about a group like its probability, containing text or its positions in the sequence are very important for the keyword spotting and will be provided directly by the derived algorithm. A low probability of the {\tt pre}-group, for example, might indicate that a letter is more likely than our separating symbol such that the spotted character sequence is only part of a larger word. Analogously, there is an equivalent group after the keyword.

Regular expressions can be very complex and the calculation of the probability of all feasible sequences can be very time consuming. We give an approximation of the most likely label sequence which we motivate theoretically and experimentally. The approximation is also fundamental to the proposed decoder since a conventional $A^*$-search suffers from a combinatorial explosion of all feasible sequences and leads to inefficient decoding times. It is developed for {\it neural networks trained by Connectionist Temporal Classification (CTC)}. Thus, CTC-trained systems are assumed all over the paper. Some of the currently most successful handwriting recognition systems were trained with CTC as shown in several competitions. To give just one example, the probably most challenging real world task is the Maurdor project which was won by A2IA in 2014 using CTC (see \cite{moysset2014a2ia}). CTC is not limited to text recognition. Recently the performance of several speech recognition systems trained with CTC equaled those of other state of the art methods (e.g. \cite{graves2014towards,sak2015learning}).

The proposed algorithm is an essential part of the award winning systems \cite{strauss2014citlab} and \cite{leifert2014citlab} which were also trained with CTC. Recently, the system reaffirmed the capability by winning the HTRtS15 competition \cite{htrts15}.

The performant connection between regular expressions and machine learning algorithms has been investigated in previous articles.
In the context of speech recognition, \cite{mohri2008speech} showed in detail how to incorporate static prior knowledge like $n$-grams or phoneme models into finite state transducers. {Although the authors exploit similar models to do the decoding, the purpose differs from ours since they model more static connections between ton, speech and language while we aim at a flexible, adaptive decoding algorithm.} 
Earlier, \cite{dupont2005links} provided a comprehensive analysis of links between probabilistic automata (i.e. automata with a probabilistic transition) and HMMs from a theoretical point of view {finally concluding -- among other -- that there is a correspondence between both models. This basically means, HMMs can be seen as the probabilistic version of finite automata.}

Some links between regular expressions, their corresponding automata and HMMs are given in \cite{krogh1998introduction}. The authors showed how to create HMMs from regular expressions to detect biological sequences. A similar but generalized approach is given in \cite{kessentini2013word}. There the authors construct a simplified HMM model for a general text line in the context of word spotting. These text line models basically consist of the keyword surrounded by space and filler models. They also proposed an enhanced model where only the
prefix or suffix of the keyword is given. This model allows a set of feasible words containing the defined prefix or suffix.

Recently, Bideault et al. published a similar approach to ours in \cite{bideaultspotting}. They proposed an HMM - BDLSTM hybrid model for word spotting exploiting regular expressions. Their model uses the posterior probability of the network as emission probability of the HMM (which means using $\p(y|x)$ as estimator for $\mathrm p(x|y)$, where $x$ is the hidden variable and $y$ is the observation). Analogously to \cite{kessentini2013word}, they build small HMM models in advance (e.g. for a keyword, for digits or letters) and combine them to a model capturing the regular expression. The authors then applied their model to keyword and ``regex'' spotting.

In contrast to the above articles, we do not make use of an HMM model. Yet in \cite{kessentini2013word}, the HMMs work only as convenient graphical model for decoding rather than as classifier. Instead of using a generative model to find the most likely sequence, our algorithm is based on the original graphical structure of the regular expressions: The finite state automata. If the automaton accepts a label sequence, it is feasible. Hence, we are able
to search in the output of a neural network for any regular expression without any previously created or trained generative model. That means as input simply serve a regular expression and the network's output matrix and the output is the most likely sequence, their probability or the capturing groups defined by the regular expression.

The remainder of this article is organized as follows: We first give a formal definition of decoding (Section \ref{sec:decoding}). In Section \ref{sec:automata}, we give a brief introduction to regular expressions and automata. Furthermore, we modify the automaton slightly to adapt it to the NN-decoding requirements. We introduce the RegEx-Decoder in Section \ref{sec:regexdecoder}. We finish with some experiments (Section \ref{sec:experiments}) and a conclusion. The appendix provides the proofs of our theorems for theoretically interested readers.
\section{{Training and decoding}}\label{sec:decoding}
This section introduces the CTC training scheme for neural networks and some basic aspects of their decoding. We mainly follow the notation of \cite{gr06}.

Let $\albet$ be the alphabet and $\albet' = \albet \cup \{\symbnac\}$ where $\symbnac$ is an artificial garbage label (also called blank) indicating that none of the labels from $\albet$ are present. We call the garbage label \emph{not a character (NaC)} in the following. An element of $\albet$ is called \emph{character} and appears in the ground truth. Sequences from $\albet^*:=\bigcup_{t\in\IN}\albet^t$ are called \emph{words}. Elements of $\albet'$ are \emph{labels} and represent different classes of the NN. Sequences of $(\albet')^*$ are called \emph{paths}. The most likely path is called \emph{best path}.
Assume a neural network which maps an input $\fett X$\footnote{In contrast to $\fett Y$, both dimensions of $\fett X$ may vary.} to a matrix $\fett Y \in \bigcup_{T=1}^\infty[0,1]^{T\times |\albet'|}$ of probabilities per position and label. I.e. $y_{t,l}$ denotes the probability for the $l$th label at position $t$. Note that we assume that $\forall t: \sum_l y_{t,l} = 1$ and $\forall t,l: y_{t,l}>0 $ throughout the paper.

To map a path $\fett \pi$ to a word $\fett z$, 
one merges consecutive identical $\pi_t$ and deletes the \nac{}s.
Let $\mathcal{F}: (\albet')^* \rightarrow \albet^*$ define the related function which maps a path to a word. More precisely: $\mathcal{F}(\fett \pi) = \mathcal{D}(\mathcal S(\fett \pi)) $ is the composition of two functions $\mathcal{D}$ and $\mathcal{S}$ where $\mathcal{S}$ deletes all consecutive identical labels and $\mathcal D$ deletes all remaining \nac{}s.

We assume that the likelihoods $y_{t,c}$ are conditionally independent for distinct $t$ given $\fett X$. Thus, the likelihood of any path $\fett \pi$ is given as
\begin{align}
\p(\fett \pi | \fett X) =& \prod_{t=1}^T y_{t,\pi_t} .\label{eq:pathProb}
\end{align}
The probability of any word $\fett z$ is then the sum of the probabilities of all paths mapping to $\fett z$:
\begin{align*}
\p(\fett z|\fett X) = \sum_{\fett \pi \in \mathcal F^{-1}(\fett z)}\p(\fett \pi | \fett X).
\end{align*}
Let $\fett{\overline z}\in (\albet')^*$ be the extension of the word $\fett z \in \albet^*$, that means we add a \nac{} before $\fett z$, after $\fett z$ and between each pair of characters. Thus, $|\fett{\overline z}| = 2 |\fett z| +1$.
Then one could calculate $\p(\fett z|\fett X)$ in an iterative manner: The forward variable $\alpha_i(t)$ denotes the probability of the prefix $z_1,\dots,z_{\lceil\frac{i-1}{2}\rceil}$ of $\fett z$ at time $t$ given $\fett X$ and, hence, $\alpha_1(t)$ denotes the probability of the empty word prefix. Thus,
\begin{align*}
\alpha_{1}(t) &= \prod_{t'=1}^ty_{t',\overline z_1} =\prod_{t'=1}^ty_{t',\symbnac}.
\end{align*}
For $t=1$, the other initial $\alpha_i(1)$ are
\begin{align}
\alpha_{2}(1) &= y_{1,\overline z_2} =y_{1,z_1}\notag\\
\alpha_{i}(1) &= 0\quad \forall i>2.\notag
\intertext{Then, probability of any prefix at time $t$ is}
\alpha_{i}(t) &= y_{t,\overline z_i}\sum_{k\in \phi_{\fett{\overline z}}(i)} \alpha_{k}(t-1) \label{eq:forwardrecursion}
\intertext{where}
\phi_{\fett{\overline z}} (k) &= \begin{cases}
\{k-1,k\} &\text{ if } {\overline z}_k= {\overline z}_{k-2} \text{ or } k=2\\
\{k-2,k-1,k\} &\text{ else}
\end{cases}.\notag
\end{align}
The probability $\p(\fett z|\fett X)$ is then equal to the sum $\alpha_{|\fett{\overline z}|}(T)+\alpha_{{|\fett{\overline z}|-1}}(T)$ of the two last forward variables at time $T$.
Analogously, one can start at $T$ and calculate the suffix probabilities:
\begin{align*}
\beta_{|\fett{\overline{Z}}|}(T) &= 1 \\
\beta_{|\fett{\overline{Z}}|}(t) &= \prod_{t'=t+1}^Ty_{t',\symbnac} \\
\beta_{|\fett{\overline{Z}}|-1}(T) &= 1\\
\beta_{i}(T) &= 0\quad \forall i<T-1\\
\beta_{i}(t) &= y_{\overline z_i,t+1}\sum_{k\in \psi_{\fett{\overline z}}(i)} \beta_{k}(t+1)
\intertext{where}
\psi_{\fett{\overline z}} (k) =& \begin{cases}
\{k+1,k\} &\text{ if } {\overline z}_k= {\overline z}_{k+2} \text{ or } k=|\fett{\overline{z}}|-1\\
\{k+2,k+1,k\} &\text{ else}
\end{cases}.
\end{align*}

\subsection{Connectionist Temporal Classification}\label{sec:ctc}
To optimize the log likelihood objective function
\[\mathcal O(\fett z, \fett X)=-\ln\p(\fett z|\fett X)\rightarrow \max,\]
Connectionist Temporal Classification uses gradient decent. Hence, we need to provide the gradient
\begin{align*}
\frac{\partial \mathcal O(\fett z, \fett X)}{\partial y_{t,l}}&= \frac{1}{\p(\fett z|\fett X)} \sum_{\substack{\pi\in\mathcal F^{-1}(\fett z)\\ \pi_t = l}} \prod_{\substack{t'=1\\t'\neq t}}^T y_{t',\pi_{t'}}
\end{align*}
for any $t\in\{1,\dots,T\}$ and $l\in\albet'$. With the above defined $\alpha$ and $\beta$,
\[\sum_{\substack{\pi\in\mathcal F^{-1}(\fett z)\\ \pi_t = l}} \prod_{\substack{t'=1\\t'\neq t}}^T y_{t',\pi_{t'}} = \sum_{\substack{i=1\\\overline z_i=l}}^{|\fett{\overline z}|}\frac{\alpha_i(t)\beta_t(t)}{y_{t,l}}.\]
Starting with $\frac{\partial \mathcal O(\fett z, \fett X)}{\partial y_{t,l}}$, the standard backpropagation algorithm propagates error into the network and optimizes its parameters.
A more detailed description can be found in \cite{gr06}.

\subsection{Decoding}
During the prediction phase, we are interested in the $\fett z \in \albet^*$ with maximizes $\p(\fett z|\fett X)$. Usually, there are conditions which allow only certain $\fett z \in \albet^*$. A common example is the condition that $\fett z$ must be an element of a certain vocabulary $\dict$.
If the allowed words are restricted to a finite vocabulary of reasonable size, one can find the most likely vocabulary item by calculating $\p(\fett z|\fett X)$ for each $\fett z \in \dict$ individually using the forward probabilities $\alpha$ as introduced above.
We call this decoding procedure \emph{string-by-string decoding} since we calculate the word probabilities one after the other.
We approximate the word likelihood by the probability of its most likely path throughout this article by replacing the sum by maximum in eq. \eqref{eq:forwardrecursion}. The most probable path yields an alignment of positions and class labels, it speeds up the calculation and -- since there is typically one dominant path -- it is a reasonable approximation to $\p(\fett z | \fett X)$. Thus,
\begin{align}
{\fett z}^*=& \argmax_{\fett z \in \dict}\max_{\fett \pi \in \mathcal{F}^{-1}(\fett z)} \p(\fett \pi | \fett X). \notag
\end{align}

\section{Regular expressions and finite automata}\label{sec:automata}

Finding the most likely label sequence following a special structure requires a tool
for describing this structure. We model them as
regular languages which have been developed to describe such complex
structures (see \cite{FrExp06}). There is a correspondence between
regular languages / regular expressions and finite-state automata -- a
model of computation of that language. We use both -- the regular expression to describe
the set of expected sequences and the automaton to exploit the transition graph during the
decoding process. This section gives a brief introduction in the field of
regular expressions and finite state automata. Readers who are already familiar with
regular expressions and finite state automata may proceed with Subsection \ref{subsec:adaptationToF}.

\begin{defi}[regular expression / regular language] 
The empty word $\varepsilon$, the empty set $\emptyset$ and $a\in \albet$
are \emph{regular expressions} denoting the \emph{regular languages}
$\{\varepsilon\}$, $\emptyset$ and $\{a\}$, respectively. If $\mathcal L(r_1)$
and $\mathcal L(r_2)$ are two regular languages defined by the regular
expressions $r_1$ and $r_2$, then also $\mathcal L(r_1)\cup \mathcal L(r_2)
= \mathcal L(r_1|r_2)$ (alternation, i.e. $r_1$ or $r_2$), $\mathcal L(r_1)
\mathcal L(r_2) = \mathcal L(r_1r_2)$ (concatenation of $r_1$ and $r_2$)
and $(\mathcal L(r_1))^*= \mathcal L(r_1^*)$ (Kleene closure, i.e. the set of all finite sequences of words from $\mathcal L(r_1)$) are regular languages. There are no
other regular languages than the above.
\end{defi}

Thus, regular expressions define languages containing specific sequences of literals from
$\albet$. Those expressions can be represented in a model of computation. This model is known as

\begin{defi}[Automaton]\label{defi:AutomatonOriginal}
The \emph{nondeterministic finite automaton (NFA)} $N$ is a 5-tuple
$(Q,\albet \cup \{\varepsilon\},\delta,q_0,F)$, where $Q$ is the finite set
of states, $\albet$ is the alphabet, $\varepsilon$ is the empty word,
$\delta: Q\times \albet \cup \{\varepsilon\} \rightarrow \mathcal P(Q)$ is
the state transition function, $q_0\in Q$ is the initial state and
$F\subseteq Q$ is the set of final states.

We call $N$ a \emph{deterministic finite automaton (DFA)} iff $\forall q\in
Q: \delta (q,\varepsilon) = \emptyset$ and $\forall q\in Q, a\in \albet:
|\delta(q,a)|\leq 1$.

\end{defi}

For any regular expression there is an NFA accepting the corresponding language and the other way around. There may be more
than one automaton accepting a regular language. Analogously, there may
be more than one regular expression describing the same language. For any specific regular expression, we will create a corresponding NFA using \emph{Thompson's Construction Algorithm} (for details see
\cite{thompson1968programming} according to which any regular expression can be converted by some combination of
the elementary NFAs depicted in Figure
\ref{fig:thompson}). An equivalent\footnote{Two finite automata are equivalent if they accept the same language.} DFA is obtained by the \emph{Subset Construction
Algorithm}.

\begin{figure}
\subfigure[][$A|B$]{
\begin{tikzpicture}[scale=0.9, transform shape]
\node[circle,draw,minimum size=6mm] (1) at (0,0) {}; \node [rectangle,
draw,rounded corners, minimum height=8mm,minimum width=10mm] (group1) at
(2,1) {$N_A$};
\node [rectangle, draw,rounded corners, minimum
height=8mm,minimum width=10mm] (group2) at (2,-1) {$N_B$};
\node[circle,draw,minimum size=6mm] (2) at (4,0) {};
\draw[->] (1)--(group1.west) node[midway, above] {$\varepsilon$};
\draw[->] (1)--(group2.west) node[midway, above] {$\varepsilon$};
\draw[->] (group1.east) -- (2) node[midway, above] {$\varepsilon$};
\draw[->] (group2.east) -- (2) node[midway, above] {$\varepsilon$};
\end{tikzpicture}\\[0.5cm]
}\hfill
\subfigure[][$A^*$]{
\begin{tikzpicture}[scale=0.9, transform shape]
\node[circle,draw,minimum size=6mm] (1) at (0,0) {}; \node [rectangle, draw,rounded
corners, minimum height=8mm,minimum width=10mm] (group1) at (2,0) {$N_A$};
\node[circle,draw,minimum size=6mm] (2) at (4,0) {}; \draw[->] (1)--(group1.west)
node[midway, above] {$\varepsilon$};
\draw[->] (group1.north east) .. controls (3,1) and (1,1) .. node[above]
{$\varepsilon$} (group1.north west); \draw[->] (group1.east) -- (2)
node[midway, above] {$\varepsilon$}; \draw[->] (1) .. controls (1,2) and
(3,2) .. node[above] {$\varepsilon$} (2);
\end{tikzpicture}\\[0.5cm]
}\hfill
\subfigure[][$AB$]{
\begin{tikzpicture}[scale=0.9, transform shape]
\node [rectangle, draw,rounded corners, minimum height=8mm,minimum
width=10mm] (group1) at (0,0) {$N_A$}; \node [rectangle, draw,rounded
corners, minimum height=8mm,minimum width=10mm] (group2) at (2,0) {$N_B$};
\draw[->] (group1.east)--(group2.west) node[midway, above]
{$\varepsilon$};
\end{tikzpicture}\\[0.5cm]
}
\caption{Schematic representation of atomic NFAs resulting from Thompson's
Algorithm. $A$ and $B$ are regular expressions and $N_A$ and $N_B$ are the related NFAs. Other quantifiers or operators can be expressed by those three.}\label{fig:thompson}
\end{figure}
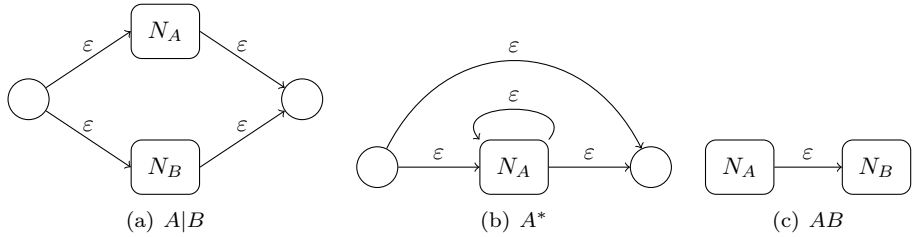

Generally, the subset construction algorithm generates a DFA with $2^n$ states
if $n$ is the number of NFA states. \cite{meyer1971economy} showed that there
are languages which also require exactly $2^n$ DFA states, i.e. the NFA is
exponentially more succinct than the DFA. Instead of using DFAs, we substitute the states of the NFA by their epsilon closure and use the resulting NFA.
That means, we delete each $\varepsilon$-transition and replace it by the next non-$\varepsilon$-transition. The resulting automaton will accept the same language as the original one.

\subsection{Adaptation to $\mathcal F$}\label{subsec:adaptationToF}
The function $\mathcal F$ (see Section \ref{sec:decoding}) maps a label sequence to a word by
deleting consecutive identical labels ($\mathcal S$) and deleting \nac{}s ($\mathcal D$). To allow optional \nac{}s between different characters
during the decoding, we extended the word $\fett z$ to $\fett{\bar z}$.
Analogously, we extend the transitions between the NFA-states
the following way:
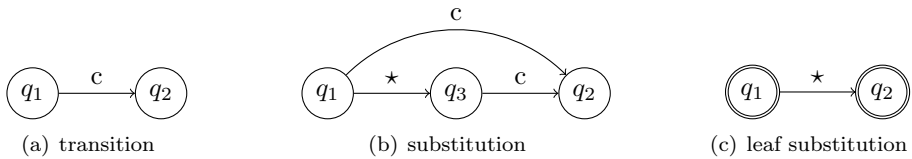
\begin{figure}
\subfigure[][transition]{
\label{fig:transition}
\tikzstyle{every state}=[inner sep=3pt,minimum size=10pt]
\begin{tikzpicture}[auto]
\node[state] (q1) at (0,0) {$q_1$}; \node[state] (q2) [right=of q1]
{$q_2$}; \path[->] (q1) edge node {c} (q2);
\end{tikzpicture}}
\hfill
\subfigure[][substitution]{
\label{fig:substitution}
\tikzstyle{every state}=[inner sep=3pt,minimum size=10pt]
\begin{tikzpicture}[auto]
\node[state] (q1) at (0,0) {$q_1$}; \node[state] (nac) [right=of q1]
{$q_3$}; \node[state] (q2) [right=of nac] {$q_2$}; \path[->] (q1) edge
[out=45,in=135] node {c} (q2); \path[->] (q1) edge node {$\symbnac$}
(nac); \path[->] (nac) edge node {c} (q2);
\end{tikzpicture}}
\hfill
\subfigure[][leaf substitution]{
\label{fig:leafsubs}
\tikzstyle{every state}=[inner sep=3pt,minimum size=10pt]
\begin{tikzpicture}[auto]
\node[state,accepting] (q1) at (0,0) {$q_1$}; \node[state,accepting]
(q2) [right=of q1] {$q_2$}; \path[->] (q1) edge node {$\symbnac$} (q2);
\end{tikzpicture}}
\caption{Illustrations for extended NFAs. Double circles represent final state.}
\end{figure}
Figure \ref{fig:transition} shows the transition which is substituted by Figure \ref{fig:substitution}. If
$q_1$ is final also $q_3$ is final. $q_1$ and $q_2$ could even be the
same state. Final leaf states (i.e. states without outgoing edges) are connected to another finale state by reading a \nac{} as shown on the Figure \ref{fig:leafsubs}.
Algorithm \ref{alg:extAutomaton} provides the pseudo code for extending the automaton. It accepts the
language {$\overline{\mathcal L}(r):=\{\overline{\fett
w}\in\mathcal{L}(\symbnac{\tt ?} {\tt w}_1 \symbnac{\tt ?w}_2 \symbnac{\tt
?} \dots \symbnac{\tt ?w}_{|{\fett{\tt w}}|}\symbnac{\tt ?})|\fett{\tt
w}\in\mathcal{L}(\fett r)\}$} of words interrupted by optional\footnote{We use the regular expression notation ({\tt ?}) to mark symbols as optional.} \nac{}s.
This is the adaptation to the $\mathcal{D}$-part of $\mathcal{F}$.

\IncMargin{1em}
\begin{algorithm}
\SetKwData{Left}{left}\SetKwData{This}{this}\SetKwData{Up}{up}
\SetKwFunction{Union}{Union}\SetKwFunction{FindCompress}{FindCompress}
\SetKwInOut{Input}{input}\SetKwInOut{Output}{output}
\Input{NFA $(Q,\albet,\delta,q_0,F)$}
\Output{Extended NFA $(\overline Q,\albet',\overline\delta,q_0,\overline F)$}
\BlankLine
$\overline Q\leftarrow Q $\;
$\overline F\leftarrow F $\;
\For{$q\in Q$}{
\emph{create new state $q'$}\;
$\overline Q\leftarrow\overline Q \cup \{q'\}$\;
$\overline\delta(q,\symbnac)\leftarrow\{q'\}$\;
\For{$a\in \albet$}{
$\overline\delta(q',a)\leftarrow\delta(q,a)$\;
}
\If{$q\in F$}{$\overline F \leftarrow \overline F \cup \{q'\}$\;}
}
\caption{extendAutomaton}\label{alg:extAutomaton}
\end{algorithm}\DecMargin{1em}

Instead of adapting $N$ also to $\mathcal{S}$, we leaf this step to the algorithm in Section \ref{sec:regexdecoder} to simplify the notation. Since $\mathcal{S}$ deletes identical consecutive labels, the continuation of a read label is left (see the ``$\cont$'' function in later sections).
We call the automaton adapted to $\mathcal{D}$ \emph{extended automaton} and symbolize it by $\overline{N}$.

\begin{ex}\label{ex:automaton}
We construct an automaton accepting the language $\mathcal{L} = \{\texttt{cat},\texttt{bat}\}$. The na\"ive alternation \texttt{cat|bat} of both words leads to an automaton with 14 states using Thomson's Construction and the above described extension. We could save 4 states and transitions by alternating only the first letters. The regular expression \texttt{(c|b)at} will generate the following automaton:
\begin{center}

\tikzstyle{every state}=[inner sep=3pt,minimum size=10pt]
\begin{tikzpicture}[shorten >=1pt,node distance=1.2cm,on grid,auto]
\node[state,initial] (q0) at (0,0) {};
\node[state] (nac1) [right=of q0] {};
\node[state] (c) [above right=of nac1] {};
\node[state] (b) [below right=of nac1] {};
\node[state] (nac2) [right=of c] {};
\node[state] (nac2b) [right=of b] {};
\node[state] (o) [below right=of nac2] {};
\node[state] (nac3) [right=of o] {};
\node[state,accepting] (u) [right=of nac3] {};
\node[state,accepting] (nac4) [right=of u] {};

\path[->] (q0) edge node[above=0mm] {$\symbnac$} (nac1);
\path[->] (q0) edge[out=45,in=180] node{c} (c);
\path[->] (nac1) edge node[above=0mm,near start] {c} (c);
\path[->] (c) edge node {$\symbnac$} (nac2);
\path[->] (c) edge[out=-45,in=180] node[above=1mm] {a} (o);

\path[->] (q0) edge[out=-45,in=-180] node[below]{b} (b);
\path[->] (nac1) edge node[above=1mm,near end] {b} (b);
\path[->] (b) edge node[below] {$\symbnac$} (nac2b);
\path[->] (b) edge[out=45,in=180] node[below]{a} (o);

\path[->] (nac2) edge node {a} (o) ;
\path[->] (nac2b) edge node[below=2mm,very near end] {a} (o) ;
\path[->] (o) edge node[below] {$\symbnac$}
(nac3); \path[->] (nac3) edge node[below] {t} (u) ;
\path[->] (o) edge
[out=45,in=135] node {t} (u) ;
\path[->] (u)edge node[below] {$\symbnac$} (nac4);

\end{tikzpicture}
\end{center}
If we aggregate the labels \texttt{c} and \texttt{b} like \texttt{[bc]at}, we could save two additional states and even 5 transitions. Thus, instead of using multiple arcs for connecting the same states but reading different labels, we aggregate them into one transition:
\begin{center}

\tikzstyle{every state}=[inner sep=3pt,minimum size=10pt]
\begin{tikzpicture}[shorten >=1pt,node distance=1.2cm,on grid,auto]
\node[state,initial] (q0) at (0,0) {}; \node[state] (nac1) [right=of q0]
{}; \node[state] (h) [right=of nac1] {};
\node[state] (nac2) [right=of h] {}; \node[state] (o) [right=of nac2] {};
\node[state] (nac3) [right=of o] {}; \node[state,accepting] (u) [right=of
nac3] {}; \node[state,accepting] (nac4) [right=of u] {};

\path[->] (q0) edge node[below] {$\symbnac$} (nac1); \path[->] (q0)
edge[out=45,in=135] node{c |b} (h); \path[->] (nac1) edge node[below] {c |b} (h);
\path[->] (h) edge node[below] {$\symbnac$} (nac2);
\path[->] (h) edge[out=45,in=135] node{a} (o);
\path[->] (nac2) edge node[below] {a} (o) ; \path[->] (o) edge node[below] {$\symbnac$}
(nac3); \path[->] (nac3) edge node[below] {t} (u) ;
\path[->] (o) edge
[out=45,in=135] node {t} (u) ;
\path[->] (u)edge node[below] {$\symbnac$} (nac4);

\end{tikzpicture}
\end{center}
Thus, there is at most one transition between any two states which reads possibly multiple labels.
Obviously, any accepted label sequence produces an emission sequence collapsing to
``cat'' or ``bat''. Note, that we need just 10 transitions where the decoding
process from Section \ref{sec:decoding} needs to calculate 7 table columns
for each word. If we add the words {\tt fat, rat, hat} to our list of accepted
words, the conventional decoding of Section \ref{sec:decoding} calculates 3.5 more
table columns than there are transitions in the automaton.
\end{ex}

\section{{Efficient decoding of regular expressions}}\label{sec:regexdecoder}
Given a regular expression $\fett r$ and the corresponding extended NFA $\overline N =
(\overline Q,\albet',\overline \delta,q_0,\overline F)$, we search for the most likely word $\fett
z^*$ in $\mathcal{L}( \fett r)$:
\[\fett z^* = \argmax_{\fett z\in\mathcal{L}(\fett r)}\max_{\fett \pi \in
\mathcal{F}^{-1}(\fett z)}\p(\fett \pi | \fett X).\] In contrast to
calculating the likelihood of every single feasible word from $\mathcal L(\fett r)$, we exploit the graphical
structure of $\overline N$ to find $\fett z^*$. This can be done very efficiently if $\overline N$
is succinct (as e.g. in Example \ref{ex:automaton}).

\subsection{$A^*$ and beam search}
In this subsection, we review two standard algorithms -- the $A^*$-search and the beam search -- which are standard approaches of decoding.

Algorithm \ref{alg:Astar} describes a na\"ive \emph{$A^*$-search} algorithm on regular expressions that returns the most likely path.
This algorithm yields the best result but it can be time consuming because of the huge number of possible paths. To cut unlikely paths, we define an upper bound $\overline \p(\fett \pi,t|\fett X)$ for the final probability $\p(\fett\pi\fett \tau|\fett X)$ of the final path $\fett \pi\fett \tau$ starting with the prefix $\fett \pi$ (position $1$ to $t$). In our experiments, we filled up $\fett \pi$ with a $\fett \beta$ suffix (i.e. $\fett \tau:=\fett \beta_{t+1:T}$) such that $\overline \p(\fett \pi,t|\fett X):= \prod_{r=1}^t y_{r,\pi_{r}} \prod_{s=t+1}^Ty_{s,\beta_s}$. Another heuristic which appears to work well in practice is to sort the prefix list $\raum L$ by $\frac{\overline\p(\fett \pi, t|\fett X)}{t}$. This sorting yields a quick first best guess such that unlikely paths can be deleted soon.
\IncMargin{1em}
\begin{algorithm}
\SetKwInOut{Input}{input}\SetKwInOut{Output}{output}
\Input{Network output $\fett Y$, extended NFA $\overline N = (\overline Q,\albet',\overline\delta,q_0,\overline F)$}
\Output{most likely feasible path $\fett \pi^*$}
\BlankLine

\For{$\gamma\in \albet'$}{
\For{$q'\in \delta(q_0,\gamma)$}{
Add $(q',\gamma,1)$ to $\mathbb L$\tcc*[r]{initialize $\raum L$}
}
}

\While {$\mathbb L$ not empty}{
$(q,\fett \pi,t) \leftarrow$ Item from $\mathbb L$ with maximum $\frac{\overline\p(\fett \pi, t|\fett X)}{t}$\;
Remove $(q,\fett \pi,t)$ from $\mathbb L$\;
\uIf{$t<T$}{
\For{$\gamma\in \albet'\setminus\{\pi_t\}$}{
\For{$q'\in \delta(q,\gamma)$}{
Add $(q',\fett \pi \gamma,t+1)$ to $\mathbb L$\;
}
}
Add $(q,\fett \pi \pi_t,t+1)$ to $\mathbb L$\tcc*[r]{cover the $\mathcal S$ part of $\mathcal F$}
}\uElseIf {$q\in F$}{
$\fett \pi^* \leftarrow \fett \pi$\;
Remove all $(q',\fett \pi',t')\in \mathbb L$ with $\overline \p(\fett \pi', t'|\fett X) <\p(\fett \pi|\fett X)$\;
}

}
\caption{$A^*$-search}\label{alg:Astar}
\end{algorithm}\DecMargin{1em}

Since the number of feasible paths grows exponentially in the worst case, there is a standard heuristic to reduce the search space called \emph{beam search}. For example in \cite{graves2014towards}, the authors introduced a beam search algorithm for efficient decoding in case of speech recognition which allows only $n$ prefixes\footnote{$n$ is called the \emph{beam width}.} at any
position. Algorithm \ref{alg:beam} contains its pseudo code adapted to our problem. Generally, beam search does not guaranty to find the optimal sequence. The given algorithm has the additional drawback that it does not even guaranty to find any feasible path at all since the final list $\raum L$ could contain only $(q,\fett \pi,T)$ with $q\not\in{F}$.
\IncMargin{1em}
\begin{algorithm}
\SetKwInOut{Input}{input}\SetKwInOut{Output}{output}
\Input{Network output $\fett Y$, extended NFA $\overline N = (\overline Q,\albet',\overline\delta,q_0,\overline F)$}
\Output{most likely feasible path $\fett \pi^*$}
\BlankLine
\For{$\gamma\in \albet'$}{
\For{$q'\in \delta(q_0,\gamma)$}{
Add $(q',\gamma,1)$ to $\mathbb L$\tcc*[r]{initialize $\raum L$}
}
}
\For{$i\leftarrow 2$ \KwTo $T$}{
$\overline{ \mathbb L}\leftarrow$ \text{ the $n$ most likely item of} $\mathbb L$\;
$\mathbb L\leftarrow \{\}$\;
\For {$(q,\fett \pi,t) \in \overline{\mathbb L}$}{
\For{$\gamma\in \albet'\setminus\{\pi_t\}$}{
\For{$q'\in Q : q'\in \delta(q,\gamma)$}{
Add $(q',\fett \pi \gamma,t+1)$ to $\mathbb L$\;
}
}
Add $(q,\fett \pi \pi_t,t+1)$ to $\mathbb L$\tcc*[r]{cover the $\mathcal S$ part of $\mathcal F$}
}
}
$\fett \pi^* \leftarrow$ $\fett \pi$ from $(q,\fett \pi,t)\in\overline{\mathbb L}$ with maximum $\p(\fett \pi|\fett X)$ and $q\in F$\;

\caption{beam search}\label{alg:beam}
\end{algorithm}\DecMargin{1em}

\subsection{RegEx-Decoder}
In this subsection, we introduce another decoding algorithm which exploits the structure of the given automaton and thus is more efficient than the $A^*$-search and guaranties -- under mild conditions -- to return the most likely path at the same time. In contrast to the token passing algorithm from \cite{young1989token}, one transition may read several input labels. We finally show that considering the three most likely labels per arc and position is sufficient. This feature allows us to preprocess the network output $\fett Y$ such that each arc only processes the most likely of their reading outputs which avoids unnecessary calculations. Additionally, we keep less paths compared to the token passing algorithm.

Let $\Pi(t,q',q)$ be the set of prefixes of $\mathcal F^{-1}(\mathcal L(\fett r))$ of length $t$ on condition that the automaton moves from state $q'$ to $q$ at position $t$.
Instead of keeping all possible prefixes, we only keep one prefix per arc, label of that arc and time point:
The probability of most likely prefix from $\Pi(t,q',q)$ is denoted by $\alpha^1_{t,q',q}$.
Multiply labeled arcs have different super scripts $i$ of $\alpha^i_{t,q',q}$ each of them corresponding to a different label of $(q',q)$.
The $\alpha_{t,q',q}^i$ can be calculated iteratively by
\begin{align*}
\alpha_{t,q',q}^i &= \max_{\substack{\fett\pi\in\Pi(t,q',q)\\\pi_t\not \in \left\{\emission^j_{t,q',q}| j < i \right\}}}\p(\fett\pi|\fett X) 
\end{align*}
where $\emission_{t,q',q}^i $ denotes the ending label $\pi_t$ of the specific $t$-prefix $\fett \pi \in \Pi(t,q',q)$
which has a likelihood of $\alpha_{t,q',q}^i$ (i.e. $\emission_{t,q',q}^i\neq\emission_{t,q',q}^j$ for $i\neq j$).\footnote{Let $\alpha^i_{t,q',q}$ be the probability of the most likely prefix $\fett \pi^i$, then $\pi^i_t=\emission_{t,q',q}^i$. Further, $\alpha^i_{t,q',q}>\alpha^{i+1}_{t,q',q}$.}
If we maximize over an empty set, we assume the result is zero. Let $\fett \alpha$ be a variable containing $\alpha^i_{t,q',q}$ for each $i,t$ and $(q',q)$.
Let $ \raum P(q) = \{q' \in \overline Q\;|\;\exists a\in \albet' :
q\in\overline\delta(q',a)\}$ be the set of predecessor states of $q$.
\begin{rem}
Let $\overline N$ be the extended automaton with respect to a regular expression $\fett r$ and
$ \fett \alpha$ as defined above.
The probability of the \emph{most likely path}
$\fett \pi^*(\fett r)$ with $\mathcal F(\fett \pi^*(\fett r))\in
\mathcal{L}(\fett r)$ is given by
\begin{align*}
\p(\fett \pi^*(\fett r)|\fett X):&=\max_{\fett z\in\mathcal{L}(r)}\max_{\fett \pi \in \mathcal{F}^{-1}(\fett z)}\p(\fett \pi | \fett X)\\
&= \max_{q\in F, q'\in\raum P(q)}\max_{\fett \pi \in \Pi(T,q',q)}\p(\fett \pi | \fett X)\\
&= \max_{q\in F, q'\in\raum P(q)}\alpha^1_{T,q',q}.
\end{align*}
\end{rem}
Thus, we only need $\alpha^1_{T,q',q}$ to calculate the likelihood of $\fett r$ with respect to $\fett Y$. Unfortunately, we need also the preceding $\fett \alpha^i$ for $i>1$ to calculate $\fett\alpha^1$.

Let
\[\likeliest_{t,q',q}^i := \argmax_{\substack{a\in\albet'\setminus\bigcup_{j<i}\left\{\likeliest^j_{t,q',q}\right\} \\ q\in \delta(q',a)}}y_{t,a}\]
be the
$i$th likely label per arc $(q',q)$ and position $t$.
This especially means $y_{t,\likeliest_{t,q',q}^1}\geq y_{t,\likeliest_{t,q',q}^2} \geq \dots$. Obviously, the initial values of $\fett \alpha$ are $\alpha^i_{1,q',q}=y_{1,\likeliest_{1,q',q}^i}$ if $q'=q_0$ and $\alpha^i_{1,q',q}=0$ else.
\begin{rem} \label{rem:dealingWithF}
Note that every non-\nac-arc represents a character or a group of equivalent characters of the regular expression.
Thus, two consecutive arcs must not read the same label $a\in\albet$ at consecutive positions since this means moving two characters forward in the accepted word. But a sequence of identical consecutive labels is mapped to one character by $\mathcal F$ which means $\mathcal F$ allows only one step forward.
Thus, if the most likely previous arc $(q'',q')$ ends on $\emission^1_{t-1,q'',q'}=\likeliest^1_{t,q',q}$,
we calculate the $t$-prefix probability by either combining the most likely $(t-1)$-prefix \emph{not reading $\likeliest_{t,q',q}^1$} with $\likeliest_{t,q',q}^1$ or we keep the most likely $(t-1)$-prefix extending it by the \emph{second most likely label} $\likeliest_{t,q',q}^2$. In the first case, we have to calculate also $\fett \alpha^2_t$ for all arcs.
\end{rem}

There are two {possible types of contributions} to calculate $\alpha^i_{t,q',q}$: We either come from a previous arc (i.e. append a new label) or we continue reading the label of the previous prefix through $(q',q)$ (i.e., stay on the arc and cover the $\mathcal S $ part of $\mathcal F$). For the most likely $t$-prefix the likelihood $\alpha^1_{t,q',q}$ is obviously calculated by
\begin{align*}
\alpha^1_{t,q',q} &= \max\{ \app(t,q',q,1), \cont(t,q',q,1)\}
\intertext{where}
\app(t,q',q,1) &= \max_{q''\in \raum P(q')}\max_{k,a} \left\{\alpha_{t-1,q'',q'}^k y_{t,a}\:|\:a\in\albet'\setminus \{\emission^k_{t-1,q'',q'}\} : q \in \overline \delta (q',a) \right\}\\
\cont(t,q',q,1) &= \max_{k} \left\{\alpha_{t-1,q',q}^k y_{t,\emission^k_{t-1,q',q}}\right\}.
\end{align*}
A straight forward generalization with the additional restriction not to read $\emission^j_{t,q',q}$ ($j<i$) leads to the general calculation schema of $\alpha^i_{t,q',q}$:
\begin{align}
\alpha^i_{t,q'q} & =\max \left\{ \app(t,q'q,i), \cont(t,q',q,i) \right\}\label{eq:alpha1Exact}
\intertext{where}
\begin{split}
\app(t,q',q,i) &= \max_{q'' \in \raum P(q')} \max_{k,a}\Big\{\alpha_{t-1,q'',q'}^ky_{t,a} | \\ &
\qquad a \in \albet'\setminus \left(\{\emission^k_{t-1,q'',q'}\}\cup\{\emission^j_{t,q',q}| j < i \}\right):\: q \in \overline \delta (q',a) \Big\}
\end{split}\label{eq:app}
\\
\cont(t,q',q,i) &= \max_k \Big\{ \alpha_{t-1,q',q}^k y_{t,\emission_{t-1,q',q}^k} | \forall j<i:\emission_{t-1,q',q}^k\neq \emission^j_{t,q',q}\Big\} \label{eq:cont}.
\end{align}
Starting from $q_0$, we now calculate $\alpha^i_{t,q',q}$ for each $i$, arc $(q',q)$ and time point $t$. The maximum $\alpha^1_{T,q',q}$ for $q\in F$ will be the maximum probability of all feasible paths. We yet even reduced the search space by keeping only one prefix probability per arc, allowed label of the specific arc and time point. This means, we have a polynomial time complexity (instead of an exponential time complexity as the $A^*$-search). More precisely,
the calculation of $\fett \alpha$ requires $\mathcal{O}(T|\albet'|\sum_{q\in Q} \sum_{q'\in \raum P(q)}| \raum P(q')| )$ multiplications
in the worst case. Although the running time seems to be cubic in the number of states $|Q|$,
in practical applications, the number of predecessors of each state is typically limited by a constant.
Thus, the expected running time is rather linear in $Q$.

The most likely path can be found via simple backtracking.

\subsubsection*{Speed-up}\label{sec:speedup}
In the following, we analyze the most likely paths of $\fett \alpha$ and speed-up the process by avoiding unnecessary calculations. The speed-up is based on two theorems which finally lead to a time complexity which is independent of the number of labels in $\albet'$.
The first theorem states that it is sufficient to know $\alpha^1_{t,q'',q'}$ and $\alpha^2_{t,q'',q'}$ for every $q''\in\raum P(q')$ to calculate both $\app(t+1,q',q,1)$ and $\app(t+1,q',q,2)$. Additionally, we only need the three most likely probabilities $y_{t+1,a}$ per arc and time step no matter how many labels allow to move from $q'$ to $q$.
\begin{theo}\label{theo:exactAlpha}
Let $\Gamma(t,k,q'',q',q):=\{\likeliest_{t,q',q}^j \;|\; j\in\{1,2,3\}\}\setminus\{\emission^k_{t-1,q'',q'}\}$ the three most likely labels without the previous ending label $\emission_{t-1,q'',q'}^i$. Then for $i=1,2$ eq. \eqref{eq:app} simplifies to
\begin{align*}
\app(t,q',q,1)
&=\max\{\alpha^k_{t-1,q'',q'}y_{t,a} | {q''\in\raum P(q'), \: k\in\{1,2\},\: a\in\Gamma(t,k,q'',q',q)}\}
\\
\begin{split}
\app(t,q',q,2)
&= \max\{\alpha^k_{t-1,q'',q'}y_{t,a} | q''\in\raum P(q'),\: k\in\{1,2\},\:\\&\qquad\qquad a\in\Gamma(t,k,q'',q',q)\setminus\{\emission^1_{t,q',q}\}\} \quad\qquad .
\end{split}
\end{align*}

\end{theo}
The proof of Theorem \ref{theo:exactAlpha} can be found in the Appendix.
An equivalent statement using the same values as in Theorem \ref{theo:exactAlpha} for the calculation of the likelihood of consecutive identical labels calculates $\cont(t,q',q,i)$ as
\begin{align}
\begin{split}
\widetilde\cont(t,q',q,1) &= \max\Big\{\alpha^k_{t-1,q',q}y_{t,\emission^k_{t-1,q',q}} \:|\:k\in\{1,2\}:\\&\qquad\qquad \emission^k_{t-1,q',q}\in\{\likeliest_{t,q',q}^j | j=1,2,3\} \Big\}\\
\widetilde\cont(t,q',q,2) &= \max\Big\{\alpha^k_{t-1,q',q}y_{t,\emission^k_{t-1,q',q}} \:|\: k\in\{1,2\}: \emission^k_{t-1,q',q}\neq\emission^1_{t,q',q} \wedge \\ &\qquad\qquad\emission^k_{t-1,q',q}\in\{\likeliest_{t,q',q}^j | j=1,2,3\}\Big\}.
\end{split}\label{eq:approxapp}
\end{align}
Although $\widetilde\cont(t,q',q,i)\neq \cont(t,q',q,i)$ in general, the conditions given in Theorem \ref{theo:approxAlpha} are sufficient to ensure that this approximation does not influence the final probability.

\begin{theo}\label{theo:approxAlpha}
Let $\fett \pi^*$ be the most likely feasible path with respect to the regular expression $\fett r$. Assume the following conditions:
\begin{enumerate}
\item \label{cond:limitedContinues} $\forall t: (\pi_t^*,\dots,\pi_{t+n}^*)=a^n\in\albet^n\Rightarrow n\leq2$ (i.e. $\fett \pi^*$ contains at most 2 consecutive identical labels from $\albet$)
\item \label{cond:likelynac} $\forall t: |\{a\in\albet \::\:y_{t,a}>y_{t,\symbnac}\}|<3$ (the \nac{} is one of the three most likely labels at each position)
\end{enumerate}
Then, {there is a $q\in F$ such that $\alpha_{T,q',q}^1 = \p(\fett \pi^*|\fett X)$} if $\fett \alpha$ is calculated using \eqref{eq:approxapp} as substitution for \eqref{eq:cont}. 
\end{theo}

Again, the proof can be found in the Appendix.

\begin{rem}\label{rem:Remarks2Approx}
Errors only appear for arcs reading more than 2 characters. We call these arcs \emph{critical}.

The conditions of Theorem \ref{theo:approxAlpha} are not unlikely to occur in Recurrent Neural Networks trained with CTC. The \nac{} is always very probable (condition 
\ref{cond:likelynac}) and the likelihoods of other labels are often very spiky (condition \ref{cond:limitedContinues}) i.e. one rarely observes more than two consecutive identical labels in the best path except for the \nac. (In \cite{bluche2015} they call this {\it the dominance of blank predictions}.)
\end{rem}

\begin{rem}
Theorem \ref{theo:exactAlpha} and \ref{theo:approxAlpha} allow us to preselect the most likely channels per arc and position. The calculations of any arc can be reduced by calculating the probability of prefixes ending on the three most likely labels of the considered arc.
\end{rem}

The calculation of $\alpha_{t+1,q',q}^i$ requires
$\alpha_{t,q'',q'}^i$ for all $q''\in\raum P(q')$ and $\alpha_{t,q',q}^i$. Thus,
there are two possible chronological orders to calculate the $\alpha_{t,q',q}^i$:
\begin{enumerate}
\item Fix $t$ and calculate $\alpha_{t,q',q}^i$ starting at $q_0$ before moving
on to $t+1$.
\item Fix $(q',q)$ and calculate $\alpha_{t,q',q}^i$ for all $t$ before moving on to
the successor states.
\end{enumerate}
We suggest the second variant mainly because of computational reasons. Finishing the calculation of
one state allows to keep the necessary values in the cache and promises a fast calculation.
However, we did not test the first variant. The downside of the second variant is that we must not allow circles
of length greater than one for the automaton $N$ (which results is circles of length
2 in the extended automaton). Otherwise, we would require information of subsequent (not yet calculated)
arcs. This restriction forbids to use the Kleene star operator in any regular expression. To allow at least the Kleene star for single characters or character groups, we calculate all transitions depicted in Figure \ref{fig:substitution} at once whenever $q_1=q_2$.

The Algorithms \ref{alg:regex} and \ref{alg:arc} show the pseudo code of the proposed algorithm.

\IncMargin{1em}
\begin{algorithm}
\SetKwInOut{Input}{input}\SetKwInOut{Output}{output}
\Input{Network output $\fett Y$, regular expression $\fett r$}
\Output{Likelihood $p=\p(\fett\pi^*(\fett r)|\fett X)$}
\BlankLine
\emph{$N \leftarrow$ createNFA($r$)}\tcp*[r]{Thompson's Construction Algorithm}
\emph{$\overline N \leftarrow$ extendAutomaton($N$)}\tcp*[r]{Algorithm \ref{alg:extAutomaton}}
\For{$q \in \mathrm {Successor}(q_0)$}{
calculate $ (\alpha_{t,q_0, q}^{1/2})_{t=1}^T$ 
\tcp*[r]{Algorithm \ref{alg:arc}}
}
$p = 0$\;
\For{$q\in F$}{
\lForEach{$q'\in\raum P(q)$}{$p\leftarrow \max\{p,\alpha_{T,q',q}^1\}$}
}
\caption{RegExDecoder}\label{alg:regex}
\end{algorithm}\DecMargin{1em}

\IncMargin{1em}
\begin{algorithm}
\SetKwInOut{Input}{input}\SetKwInOut{Output}{output}
\Input{likelihoods $\fett\alpha$, ending labels $\fett\emission$, extended NFA $\overline N$ }
\BlankLine
\If{all $\alpha_{t,q'',q'}^i$ with $q''\in\raum P(q')$ are calculated} {
\For{$i\leftarrow 1$ \KwTo $2$}{
\eIf{$q'==q_0$}{
$\alpha^i_{1,q',q}\leftarrow y_{1,\likeliest^i_{1,q',q}}$\;
$\emission^i_{1,q',q}\leftarrow\likeliest^i_{1,q',q}$\;
}{
$\alpha^i_{1,q',q}\leftarrow 0$\;
}
}
\For{$t\leftarrow 2$ \KwTo $T$}{
$\app^1_t\leftarrow 0$\;
$\cont^1_t\leftarrow 0$\;
\For{$k=1$ \KwTo$2$}{
\For{$j=1$ \KwTo$3$}{
\For{$q''\in\raum P(q')$}{
\If{$\emission^k_{t-1,q'',q'}\neq\likeliest_{t,q',q}^j $}{
$\app^1_t\leftarrow \max\left\{\app^1_t\:,\:\alpha^k_{t-1,q'',q'}y_{t,\likeliest_{t,q',q}^j} \right\}$\;
}
}
\If{$\emission^k_{t-1,q',q}== \gamma^j_{t,q',q}$ }{
$\cont^1_{t}\leftarrow\max\left\{\cont^1_t\:,\:\alpha^k_{t-1,q',q}y_{t,\gamma^j_{t,q',q}}\right\}$\;
}
}
}
$\alpha^1_{t,q',q}\leftarrow\max\{\app^1_t,\cont^1_{t}\}$\;
\emph{$\mathrm {update}(\emission^1_{t,q',q})$}\tcp*[r]{set $\emission^1_{t,q',q}$ to the maximizing $\gamma^i_{t,q',q}$}
\tcp{calculate $\alpha_{t,q',q}^2$ analogously with the additional constraint not to read $\emission^1_{t,q',q}$}
}
\lForEach{$\hat q \in \mathrm {Successor}(q)$}{calculate $(\alpha^{1/2}_{t,q,\hat q})_{t=1}^T$}
}

\caption{calculate $(\alpha_{t,q', q}^{1/2})_{t=1}^T$}\label{alg:arc}
\end{algorithm}\DecMargin{1em}
\subsubsection*{Capturing Groups}
As already mentioned, information about a part of the regular expression
can be crucial. In case of keyword spotting for example, the likelihood of the keyword
determines whether or not the current spot is accepted. But also the likelihood of labels next to the keyword are important to decide whether or not the spotted word is only a part of a larger word. To connect parts of the regular expression with parts of the automaton, we take advantage of the notation of capturing groups:

A \emph{capturing group} $g$ of a regular expression $\fett r$ is a consecutive part within
a pair of parentheses. Thus, the group is related to certain arcs of the automaton. Hence, only if the most likely path related to $\fett r$ makes use of any arc related to $g$, $g$ captures some part of the current output $\fett Y$.
Then, the captured label sequence is the part of most likely path $\fett \pi^*(\fett r)$ read by the subautomaton related to $g$.
In a straight forward way, one calculates the probability or the bounds (start and end position) of $g$ according to $\fett \pi^*(\fett r)$.\footnote{Since the \nac{} is not part of the regular expression, one may decide whether or not
the likelihood calculation and the optimal path include the starting and tailing \nac -labels. }

\subsubsection*{Vocabularies}\label{sec:vocabularies}
Typically, a decoding process includes one or more vocabularies. The regular expression of such a vocabulary can be
expressed as an alternation of words. The optimal automaton accepting a collection of words
is well know: The \emph{deterministic,
acyclic finite state automata (DAFSA)}. There are very efficient algorithms
for a constructing a corresponding minimal DAFSA (see \cite{daciuk2000incremental}). The number
of arcs decreases dramatically compared to alternating the vocabulary words na\"ively.

Nevertheless, the number of arcs increases strongly for large vocabularies such that a fast and effective
decoding process is impossible.

\section{Experiments}\label{sec:experiments}
The aim of this section is to show that the decoding works properly and fast.
We show that the Algorithms \ref{alg:regex} and \ref{alg:arc} work correctly in practical applications and analyze situations when it fails. We compare our approximation of eq. \eqref{eq:approxapp} with the exact most likely path. Further applications of the RegEx-Decoder can be found in \cite{strauss2014citlab} and \cite{leifert2014citlab}.

We did all time statistics on a laptop with Intel i7-4940MX 3.10GHz CPU, 32GB RAM and SSD.

\subsection{Text recognition}\label{subsec:htrts}
First, we show that our approximation is reasonable for practical applications such as the HTRtS competition from the ICFHR2014 (see \cite{htrts2014}). The data consists of 400 handwritten pages. We train on 350 pages and validate on 50 pages. The validation set is also used to evaluate the decoding. Each page consists of several lines of text including words, punctuations, numbers and symbols. The neural network used in \cite{strauss2014citlab} generates the output matrices. We compare the most likely word of a vocabulary obtained by the RegEx-Decoder\footnote{The automaton is generated using the strategy of Section \ref{sec:vocabularies}.} with the result of the string-by-string decoding from Section \ref{sec:decoding}. For this purpose, the RegEx-Decoder is used to splits these matrices into regions of words and region containing spaces, numbers etc. The evaluation is done on the resulting 4657 submatrices representing the word regions. These matrices correspond to outputs of subimages of single words.
We use two vocabularies: one containing 9273 words (generated from HTRtS data) and one containing 21698 words (a modern, general vocabulary made from two million English sentences from \url{http://corpora.uni-leipzig.de/}). 

\begin{table}
\caption{Statistics over the text recognition experiment: ``size'' denotes the number of words of
the vocabulary, ``\# arcs'' denotes the number of arcs in automaton and ``\# critical arcs''
denotes the number of arcs which read more than 2 labels. ``greatest deviation - absolute'' denotes the difference
between the exact negative logarithmic probability and the result of the RegEx-Decoder.
The ``greatest deviation - relative'' is the deviation divided by the exact absolute logarithmic probability.}\label{tab:htrts}
\begin{center}
\begin{tabular}{rccccc}
\hline\hline
vocabulary & size & \multicolumn{2}{c}{arcs}& \multicolumn{2}{c}{greatest deviation}\\
&&\# total & \# critical & absolute & relative \\\hline
HTRtS & 9273 & 12398 & 12 & 9.95E-14 & 2.1E-12\\
general English & 21698 & 25997 & 32 & 9.95E-14 & 2.1E-12 \\\hline\hline
\end{tabular}
\end{center}
\end{table}

Table \ref{tab:htrts} shows the deviation of the negative logarithmic likelihood of the RegEx-Decoder and the exact decoding. Clearly, the deviation is negligible. There is an intersection of both vocabularies which includes especially the most frequent words. Thus, it is not surprising that both vocabularies show the same deviation since both extrema (the greatest absolute and relative deviation) appear for the same words (``General'' and ``of''). Since the number of critical arcs is very small, we expected a small divergence due to our approximation. In fact there is \emph{no additional confusion of words} because of our approximation. 
Thus, the experiment shows that the approximation of Theorem \ref{theo:approxAlpha} can be applied in practical applications with few critical arcs.

We evaluated the impact of the decoder empirically on the HTRtS15 test set. We decreased the word error rate by 3 percentage points compared to the best path decoding of the entire line (from 50.89\% to 48.06\%) just by defining an appropriate regular expression for the expected line structure without any vocabulary. Including a vocabulary, we further decreased the WER to 33.90\%.

\subsection{Number recognition}
The next experiment involves artificially generated writings and investigates the correctness of Alg. \ref{alg:regex} in case of a relatively large number of critical arcs.
By Remark \ref{rem:Remarks2Approx}, we know that errors only appear for arcs reading more than two labels. We enforce this condition by searching only for digits. Thus, every arc not reading \nac{}s is critical since these arcs read more than two labels. To enforce further continuation errors\footnote{Remember that errors only happen while calculating $\cont(t,q',q,i)$.}, we vary the number of digits actually depicted in the image while the search pattern remains 3 to 5 digits (i.e. the regular expression is {\tt [0-9]\{3,5\}}). If the number of digits is greater than 5, the decoder has to suppress emissions which also promotes errors. 

We vary the number of digits from 4 to 9. For each number of digits, we generate 10,000 synthetic writings.
The digits are narrowly written to enforce further confusions (see Figure \ref{fig:digits}). The resulting images work as input to four neural networks with different number recognition expertise. We will compare the decoding results over the output matrices generated by these networks. The RegEx-Decoder searches in these output matrices for the most likely number with 3 to 5 digits. The resulting number and probability is compared with the most likely number resulting from a traditional string-by-string decoding as in Section \ref{sec:decoding} using a vocabulary of all numbers with 3 to 5 digits. Any difference in the resulting optimal path (but not its probability) is regarded as an error.

\begin{figure}
\includegraphics[width=5cm]{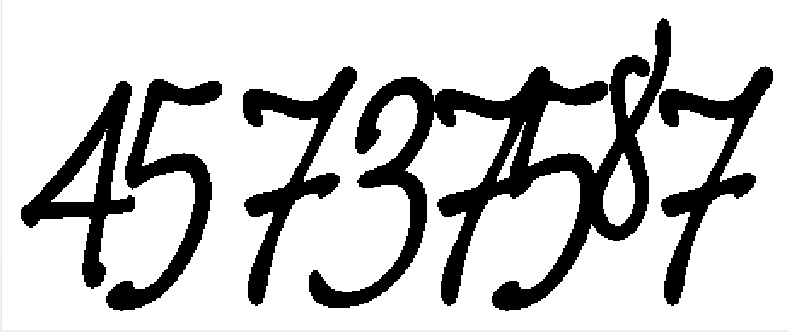}\hfill
\includegraphics[width=7cm]{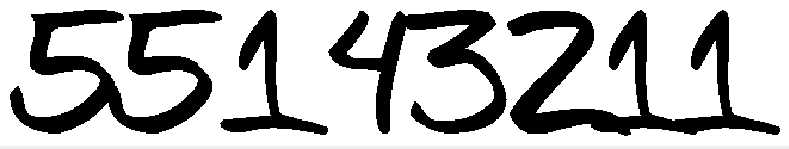}
\caption{Artificial writings of two numbers for the number recognition task.}\label{fig:digits}
\end{figure}

\begin{table}
\begin{center}
\caption{Number recognition task: Number of differences in the most likely paths of the RegEx-Decoder and the exact decoding for different neural nets and different number of digits in the image but constant regular expression of {\tt [0-9]\{3,5\}}.
}\label{tab:digits}
\begin{tabular}{lcccccc}
\hline\hline
&4 &5 &6 &7 &8 &9\\\hline
net1 &0 &0 &1 &12 &9 &27\\
net2 &0 &0 &24 &27 &40 &40\\
net3 &0 &0 &6 &3 &4 &4\\
net4 &0 &1 &4 &4 &7 &7\\\hline\hline
\end{tabular}
\end{center}
\end{table}

Table \ref{tab:digits} shows the errors per network
and digits in the image. The more the algorithm is forced to suppress digits the more errors occur. For 4 and 5 digits there is no force to suppress any written digit since the corresponding automaton is allowed to accept the ground truth. The errors are negligible in this case. However, although there are almost no errors in the resulting path, there are small differences between the probability of the string-by-string decoding and the RegEx-Decoder.
From 6 to 9, digits there are already significantly many errors.

Even if there is a relatively high number of critical arcs, there will be only little error if the regular expression fits to the image content. If it does not fit to the number of digits in the image there will be a high risk of generating additional confusion errors because of our approximation.
However, even under exact decoding the best feasible path then has a very low probability which can only by further underestimated by the approximation. Hence the approximation will likely not be harmful here since the decoding process result can either be rejected immediately or it is unlikely to be of any significance in downstream processing steps.

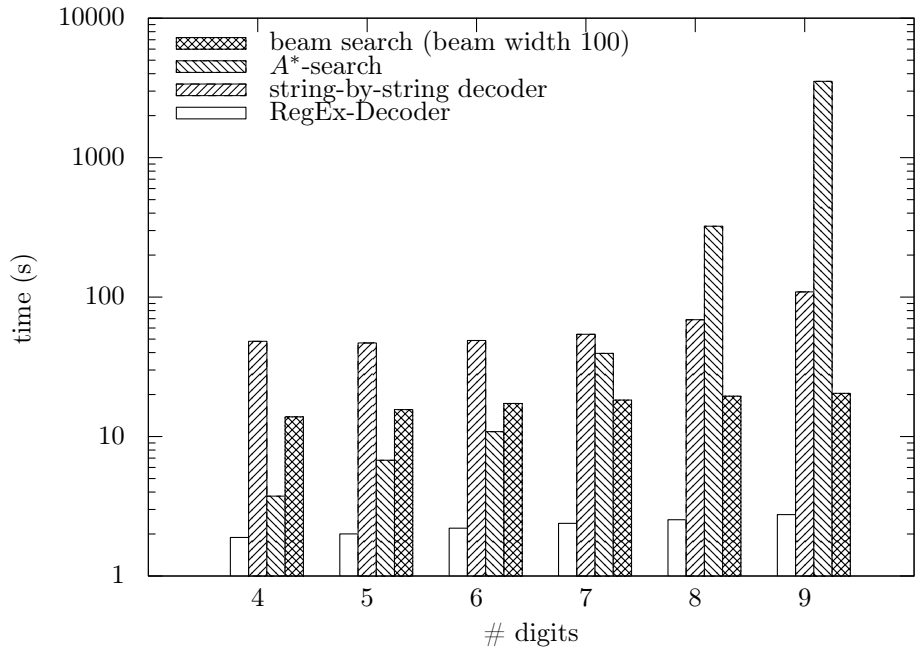
\begin{figure}
\begin{center}

\begin{tikzpicture}[gnuplot]
\path (0.000,0.000) rectangle (12.500,8.750);
\gpcolor{color=gp lt color border}
\gpsetlinetype{gp lt border}
\gpsetlinewidth{1.00}
\draw[gp path] (1.872,0.985)--(2.052,0.985);
\draw[gp path] (11.947,0.985)--(11.767,0.985);
\node[gp node right] at (1.688,0.985) { 1};
\draw[gp path] (1.872,1.542)--(1.962,1.542);
\draw[gp path] (11.947,1.542)--(11.857,1.542);
\draw[gp path] (1.872,1.867)--(1.962,1.867);
\draw[gp path] (11.947,1.867)--(11.857,1.867);
\draw[gp path] (1.872,2.098)--(1.962,2.098);
\draw[gp path] (11.947,2.098)--(11.857,2.098);
\draw[gp path] (1.872,2.277)--(1.962,2.277);
\draw[gp path] (11.947,2.277)--(11.857,2.277);
\draw[gp path] (1.872,2.424)--(1.962,2.424);
\draw[gp path] (11.947,2.424)--(11.857,2.424);
\draw[gp path] (1.872,2.548)--(1.962,2.548);
\draw[gp path] (11.947,2.548)--(11.857,2.548);
\draw[gp path] (1.872,2.655)--(1.962,2.655);
\draw[gp path] (11.947,2.655)--(11.857,2.655);
\draw[gp path] (1.872,2.749)--(1.962,2.749);
\draw[gp path] (11.947,2.749)--(11.857,2.749);
\draw[gp path] (1.872,2.834)--(2.052,2.834);
\draw[gp path] (11.947,2.834)--(11.767,2.834);
\node[gp node right] at (1.688,2.834) { 10};
\draw[gp path] (1.872,3.391)--(1.962,3.391);
\draw[gp path] (11.947,3.391)--(11.857,3.391);
\draw[gp path] (1.872,3.716)--(1.962,3.716);
\draw[gp path] (11.947,3.716)--(11.857,3.716);
\draw[gp path] (1.872,3.947)--(1.962,3.947);
\draw[gp path] (11.947,3.947)--(11.857,3.947);
\draw[gp path] (1.872,4.126)--(1.962,4.126);
\draw[gp path] (11.947,4.126)--(11.857,4.126);
\draw[gp path] (1.872,4.273)--(1.962,4.273);
\draw[gp path] (11.947,4.273)--(11.857,4.273);
\draw[gp path] (1.872,4.397)--(1.962,4.397);
\draw[gp path] (11.947,4.397)--(11.857,4.397);
\draw[gp path] (1.872,4.504)--(1.962,4.504);
\draw[gp path] (11.947,4.504)--(11.857,4.504);
\draw[gp path] (1.872,4.598)--(1.962,4.598);
\draw[gp path] (11.947,4.598)--(11.857,4.598);
\draw[gp path] (1.872,4.683)--(2.052,4.683);
\draw[gp path] (11.947,4.683)--(11.767,4.683);
\node[gp node right] at (1.688,4.683) { 100};
\draw[gp path] (1.872,5.240)--(1.962,5.240);
\draw[gp path] (11.947,5.240)--(11.857,5.240);
\draw[gp path] (1.872,5.565)--(1.962,5.565);
\draw[gp path] (11.947,5.565)--(11.857,5.565);
\draw[gp path] (1.872,5.796)--(1.962,5.796);
\draw[gp path] (11.947,5.796)--(11.857,5.796);
\draw[gp path] (1.872,5.975)--(1.962,5.975);
\draw[gp path] (11.947,5.975)--(11.857,5.975);
\draw[gp path] (1.872,6.122)--(1.962,6.122);
\draw[gp path] (11.947,6.122)--(11.857,6.122);
\draw[gp path] (1.872,6.246)--(1.962,6.246);
\draw[gp path] (11.947,6.246)--(11.857,6.246);
\draw[gp path] (1.872,6.353)--(1.962,6.353);
\draw[gp path] (11.947,6.353)--(11.857,6.353);
\draw[gp path] (1.872,6.447)--(1.962,6.447);
\draw[gp path] (11.947,6.447)--(11.857,6.447);
\draw[gp path] (1.872,6.532)--(2.052,6.532);
\draw[gp path] (11.947,6.532)--(11.767,6.532);
\node[gp node right] at (1.688,6.532) { 1000};
\draw[gp path] (1.872,7.089)--(1.962,7.089);
\draw[gp path] (11.947,7.089)--(11.857,7.089);
\draw[gp path] (1.872,7.414)--(1.962,7.414);
\draw[gp path] (11.947,7.414)--(11.857,7.414);
\draw[gp path] (1.872,7.645)--(1.962,7.645);
\draw[gp path] (11.947,7.645)--(11.857,7.645);
\draw[gp path] (1.872,7.824)--(1.962,7.824);
\draw[gp path] (11.947,7.824)--(11.857,7.824);
\draw[gp path] (1.872,7.971)--(1.962,7.971);
\draw[gp path] (11.947,7.971)--(11.857,7.971);
\draw[gp path] (1.872,8.095)--(1.962,8.095);
\draw[gp path] (11.947,8.095)--(11.857,8.095);
\draw[gp path] (1.872,8.202)--(1.962,8.202);
\draw[gp path] (11.947,8.202)--(11.857,8.202);
\draw[gp path] (1.872,8.296)--(1.962,8.296);
\draw[gp path] (11.947,8.296)--(11.857,8.296);
\draw[gp path] (1.872,8.381)--(2.052,8.381);
\draw[gp path] (11.947,8.381)--(11.767,8.381);
\node[gp node right] at (1.688,8.381) { 10000};
\draw[gp path] (3.311,0.985)--(3.311,1.165);
\draw[gp path] (3.311,8.381)--(3.311,8.201);
\node[gp node center] at (3.311,0.677) {4};
\draw[gp path] (4.751,0.985)--(4.751,1.165);
\draw[gp path] (4.751,8.381)--(4.751,8.201);
\node[gp node center] at (4.751,0.677) {5};
\draw[gp path] (6.190,0.985)--(6.190,1.165);
\draw[gp path] (6.190,8.381)--(6.190,8.201);
\node[gp node center] at (6.190,0.677) {6};
\draw[gp path] (7.629,0.985)--(7.629,1.165);
\draw[gp path] (7.629,8.381)--(7.629,8.201);
\node[gp node center] at (7.629,0.677) {7};
\draw[gp path] (9.068,0.985)--(9.068,1.165);
\draw[gp path] (9.068,8.381)--(9.068,8.201);
\node[gp node center] at (9.068,0.677) {8};
\draw[gp path] (10.508,0.985)--(10.508,1.165);
\draw[gp path] (10.508,8.381)--(10.508,8.201);
\node[gp node center] at (10.508,0.677) {9};
\draw[gp path] (1.872,8.381)--(1.872,0.985)--(11.947,0.985)--(11.947,8.381)--cycle;
\node[gp node center,rotate=-270] at (0.246,4.683) {time (s)};
\node[gp node center] at (6.909,0.215) {\# digits};
\node[gp node left] at (3.340,7.123) {RegEx-Decoder};
\def\gpfillpath{(2.240,7.046)--(3.156,7.046)--(3.156,7.200)--(2.240,7.200)--cycle}
\gpfill{color=gpbgfillcolor} \gpfillpath;
\gpfill{rgb color={0.000,0.000,0.000},gp pattern 0,pattern color=.} \gpfillpath;
\gpcolor{rgb color={0.000,0.000,0.000}}
\gpsetlinetype{gp lt plot 0}
\draw[gp path] (2.240,7.046)--(3.156,7.046)--(3.156,7.200)--(2.240,7.200)--cycle;
\def\gpfillpath{(2.951,0.985)--(3.192,0.985)--(3.192,1.497)--(2.951,1.497)--cycle}
\gpfill{color=gpbgfillcolor} \gpfillpath;
\gpfill{rgb color={0.000,0.000,0.000},gp pattern 0,pattern color=.} \gpfillpath;
\draw[gp path] (2.951,0.985)--(2.951,1.496)--(3.191,1.496)--(3.191,0.985)--cycle;
\def\gpfillpath{(4.391,0.985)--(4.632,0.985)--(4.632,1.544)--(4.391,1.544)--cycle}
\gpfill{color=gpbgfillcolor} \gpfillpath;
\gpfill{rgb color={0.000,0.000,0.000},gp pattern 0,pattern color=.} \gpfillpath;
\draw[gp path] (4.391,0.985)--(4.391,1.543)--(4.631,1.543)--(4.631,0.985)--cycle;
\def\gpfillpath{(5.830,0.985)--(6.071,0.985)--(6.071,1.620)--(5.830,1.620)--cycle}
\gpfill{color=gpbgfillcolor} \gpfillpath;
\gpfill{rgb color={0.000,0.000,0.000},gp pattern 0,pattern color=.} \gpfillpath;
\draw[gp path] (5.830,0.985)--(5.830,1.619)--(6.070,1.619)--(6.070,0.985)--cycle;
\def\gpfillpath{(7.269,0.985)--(7.510,0.985)--(7.510,1.684)--(7.269,1.684)--cycle}
\gpfill{color=gpbgfillcolor} \gpfillpath;
\gpfill{rgb color={0.000,0.000,0.000},gp pattern 0,pattern color=.} \gpfillpath;
\draw[gp path] (7.269,0.985)--(7.269,1.683)--(7.509,1.683)--(7.509,0.985)--cycle;
\def\gpfillpath{(8.709,0.985)--(8.949,0.985)--(8.949,1.732)--(8.709,1.732)--cycle}
\gpfill{color=gpbgfillcolor} \gpfillpath;
\gpfill{rgb color={0.000,0.000,0.000},gp pattern 0,pattern color=.} \gpfillpath;
\draw[gp path] (8.709,0.985)--(8.709,1.731)--(8.948,1.731)--(8.948,0.985)--cycle;
\def\gpfillpath{(10.148,0.985)--(10.389,0.985)--(10.389,1.800)--(10.148,1.800)--cycle}
\gpfill{color=gpbgfillcolor} \gpfillpath;
\gpfill{rgb color={0.000,0.000,0.000},gp pattern 0,pattern color=.} \gpfillpath;
\draw[gp path] (10.148,0.985)--(10.148,1.799)--(10.388,1.799)--(10.388,0.985)--cycle;
\gpcolor{color=gp lt color border}
\node[gp node left] at (3.340,7.431) {string-by-string decoder};
\def\gpfillpath{(2.240,7.354)--(3.156,7.354)--(3.156,7.508)--(2.240,7.508)--cycle}
\gpfill{color=gpbgfillcolor} \gpfillpath;
\gpfill{rgb color={0.000,0.000,0.000},gp pattern 1,pattern color=.} \gpfillpath;
\gpcolor{rgb color={0.000,0.000,0.000}}
\gpsetlinetype{gp lt plot 1}
\draw[gp path] (2.240,7.354)--(3.156,7.354)--(3.156,7.508)--(2.240,7.508)--cycle;
\def\gpfillpath{(3.191,0.985)--(3.432,0.985)--(3.432,4.097)--(3.191,4.097)--cycle}
\gpfill{color=gpbgfillcolor} \gpfillpath;
\gpfill{rgb color={0.000,0.000,0.000},gp pattern 1,pattern color=.} \gpfillpath;
\draw[gp path] (3.191,0.985)--(3.191,4.096)--(3.431,4.096)--(3.431,0.985)--cycle;
\def\gpfillpath{(4.631,0.985)--(4.872,0.985)--(4.872,4.075)--(4.631,4.075)--cycle}
\gpfill{color=gpbgfillcolor} \gpfillpath;
\gpfill{rgb color={0.000,0.000,0.000},gp pattern 1,pattern color=.} \gpfillpath;
\draw[gp path] (4.631,0.985)--(4.631,4.074)--(4.871,4.074)--(4.871,0.985)--cycle;
\def\gpfillpath{(6.070,0.985)--(6.311,0.985)--(6.311,4.108)--(6.070,4.108)--cycle}
\gpfill{color=gpbgfillcolor} \gpfillpath;
\gpfill{rgb color={0.000,0.000,0.000},gp pattern 1,pattern color=.} \gpfillpath;
\draw[gp path] (6.070,0.985)--(6.070,4.107)--(6.310,4.107)--(6.310,0.985)--cycle;
\def\gpfillpath{(7.509,0.985)--(7.750,0.985)--(7.750,4.190)--(7.509,4.190)--cycle}
\gpfill{color=gpbgfillcolor} \gpfillpath;
\gpfill{rgb color={0.000,0.000,0.000},gp pattern 1,pattern color=.} \gpfillpath;
\draw[gp path] (7.509,0.985)--(7.509,4.189)--(7.749,4.189)--(7.749,0.985)--cycle;
\def\gpfillpath{(8.948,0.985)--(9.189,0.985)--(9.189,4.383)--(8.948,4.383)--cycle}
\gpfill{color=gpbgfillcolor} \gpfillpath;
\gpfill{rgb color={0.000,0.000,0.000},gp pattern 1,pattern color=.} \gpfillpath;
\draw[gp path] (8.948,0.985)--(8.948,4.382)--(9.188,4.382)--(9.188,0.985)--cycle;
\def\gpfillpath{(10.388,0.985)--(10.629,0.985)--(10.629,4.752)--(10.388,4.752)--cycle}
\gpfill{color=gpbgfillcolor} \gpfillpath;
\gpfill{rgb color={0.000,0.000,0.000},gp pattern 1,pattern color=.} \gpfillpath;
\draw[gp path] (10.388,0.985)--(10.388,4.751)--(10.628,4.751)--(10.628,0.985)--cycle;
\gpcolor{color=gp lt color border}
\node[gp node left] at (3.340,7.739) {$A^*$-search};
\def\gpfillpath{(2.240,7.662)--(3.156,7.662)--(3.156,7.816)--(2.240,7.816)--cycle}
\gpfill{color=gpbgfillcolor} \gpfillpath;
\gpfill{rgb color={0.000,0.000,0.000},gp pattern 2,pattern color=.} \gpfillpath;
\gpcolor{rgb color={0.000,0.000,0.000}}
\gpsetlinetype{gp lt plot 2}
\draw[gp path] (2.240,7.662)--(3.156,7.662)--(3.156,7.816)--(2.240,7.816)--cycle;
\def\gpfillpath{(3.431,0.985)--(3.672,0.985)--(3.672,2.044)--(3.431,2.044)--cycle}
\gpfill{color=gpbgfillcolor} \gpfillpath;
\gpfill{rgb color={0.000,0.000,0.000},gp pattern 2,pattern color=.} \gpfillpath;
\draw[gp path] (3.431,0.985)--(3.431,2.043)--(3.671,2.043)--(3.671,0.985)--cycle;
\def\gpfillpath{(4.871,0.985)--(5.111,0.985)--(5.111,2.520)--(4.871,2.520)--cycle}
\gpfill{color=gpbgfillcolor} \gpfillpath;
\gpfill{rgb color={0.000,0.000,0.000},gp pattern 2,pattern color=.} \gpfillpath;
\draw[gp path] (4.871,0.985)--(4.871,2.519)--(5.110,2.519)--(5.110,0.985)--cycle;
\def\gpfillpath{(6.310,0.985)--(6.551,0.985)--(6.551,2.899)--(6.310,2.899)--cycle}
\gpfill{color=gpbgfillcolor} \gpfillpath;
\gpfill{rgb color={0.000,0.000,0.000},gp pattern 2,pattern color=.} \gpfillpath;
\draw[gp path] (6.310,0.985)--(6.310,2.898)--(6.550,2.898)--(6.550,0.985)--cycle;
\def\gpfillpath{(7.749,0.985)--(7.990,0.985)--(7.990,3.937)--(7.749,3.937)--cycle}
\gpfill{color=gpbgfillcolor} \gpfillpath;
\gpfill{rgb color={0.000,0.000,0.000},gp pattern 2,pattern color=.} \gpfillpath;
\draw[gp path] (7.749,0.985)--(7.749,3.936)--(7.989,3.936)--(7.989,0.985)--cycle;
\def\gpfillpath{(9.188,0.985)--(9.429,0.985)--(9.429,5.623)--(9.188,5.623)--cycle}
\gpfill{color=gpbgfillcolor} \gpfillpath;
\gpfill{rgb color={0.000,0.000,0.000},gp pattern 2,pattern color=.} \gpfillpath;
\draw[gp path] (9.188,0.985)--(9.188,5.622)--(9.428,5.622)--(9.428,0.985)--cycle;
\def\gpfillpath{(10.628,0.985)--(10.869,0.985)--(10.869,7.544)--(10.628,7.544)--cycle}
\gpfill{color=gpbgfillcolor} \gpfillpath;
\gpfill{rgb color={0.000,0.000,0.000},gp pattern 2,pattern color=.} \gpfillpath;
\draw[gp path] (10.628,0.985)--(10.628,7.543)--(10.868,7.543)--(10.868,0.985)--cycle;
\gpcolor{color=gp lt color border}
\node[gp node left] at (3.340,8.047) {beam search (beam width 100)};
\def\gpfillpath{(2.240,7.970)--(3.156,7.970)--(3.156,8.124)--(2.240,8.124)--cycle}
\gpfill{color=gpbgfillcolor} \gpfillpath;
\gpfill{rgb color={0.000,0.000,0.000},gp pattern 3,pattern color=.} \gpfillpath;
\gpcolor{rgb color={0.000,0.000,0.000}}
\gpsetlinetype{gp lt plot 3}
\draw[gp path] (2.240,7.970)--(3.156,7.970)--(3.156,8.124)--(2.240,8.124)--cycle;
\def\gpfillpath{(3.671,0.985)--(3.912,0.985)--(3.912,3.095)--(3.671,3.095)--cycle}
\gpfill{color=gpbgfillcolor} \gpfillpath;
\gpfill{rgb color={0.000,0.000,0.000},gp pattern 3,pattern color=.} \gpfillpath;
\draw[gp path] (3.671,0.985)--(3.671,3.094)--(3.911,3.094)--(3.911,0.985)--cycle;
\def\gpfillpath{(5.110,0.985)--(5.351,0.985)--(5.351,3.191)--(5.110,3.191)--cycle}
\gpfill{color=gpbgfillcolor} \gpfillpath;
\gpfill{rgb color={0.000,0.000,0.000},gp pattern 3,pattern color=.} \gpfillpath;
\draw[gp path] (5.110,0.985)--(5.110,3.190)--(5.350,3.190)--(5.350,0.985)--cycle;
\def\gpfillpath{(6.550,0.985)--(6.791,0.985)--(6.791,3.273)--(6.550,3.273)--cycle}
\gpfill{color=gpbgfillcolor} \gpfillpath;
\gpfill{rgb color={0.000,0.000,0.000},gp pattern 3,pattern color=.} \gpfillpath;
\draw[gp path] (6.550,0.985)--(6.550,3.272)--(6.790,3.272)--(6.790,0.985)--cycle;
\def\gpfillpath{(7.989,0.985)--(8.230,0.985)--(8.230,3.318)--(7.989,3.318)--cycle}
\gpfill{color=gpbgfillcolor} \gpfillpath;
\gpfill{rgb color={0.000,0.000,0.000},gp pattern 3,pattern color=.} \gpfillpath;
\draw[gp path] (7.989,0.985)--(7.989,3.317)--(8.229,3.317)--(8.229,0.985)--cycle;
\def\gpfillpath{(9.428,0.985)--(9.669,0.985)--(9.669,3.370)--(9.428,3.370)--cycle}
\gpfill{color=gpbgfillcolor} \gpfillpath;
\gpfill{rgb color={0.000,0.000,0.000},gp pattern 3,pattern color=.} \gpfillpath;
\draw[gp path] (9.428,0.985)--(9.428,3.369)--(9.668,3.369)--(9.668,0.985)--cycle;
\def\gpfillpath{(10.868,0.985)--(11.108,0.985)--(11.108,3.407)--(10.868,3.407)--cycle}
\gpfill{color=gpbgfillcolor} \gpfillpath;
\gpfill{rgb color={0.000,0.000,0.000},gp pattern 3,pattern color=.} \gpfillpath;
\draw[gp path] (10.868,0.985)--(10.868,3.406)--(11.107,3.406)--(11.107,0.985)--cycle;
\gpcolor{color=gp lt color border}
\gpsetlinetype{gp lt border}
\draw[gp path] (1.872,8.381)--(1.872,0.985)--(11.947,0.985)--(11.947,8.381)--cycle;
\gpdefrectangularnode{gp plot 1}{\pgfpoint{1.872cm}{0.985cm}}{\pgfpoint{11.947cm}{8.381cm}}
\end{tikzpicture}
\end{center}
\caption{Decoding times for the number recognition task (10,000 output matrices) averaged over all four networks.}
\label{fig:time}
\end{figure}

Figure \ref{fig:time} shows the required decoding time for the above network outputs and regular expression. The RegEx-Decoder needs between 0.19 ms and 0.28 ms per network output on average. The conventional string-by-string decoding needs at least 4.68 ms per network output since it has to calculate the probabilities of more or less all numbers with the specific number of digits under consideration. To speed up the decoding time, this decoding method reuses already calculated probabilities whenever the beginnings are the same\footnote{I.e. {$12345$} and {$12346$} share all probabilities for the prefix {$1234$}.}. Additionally, it stops the calculation of paths if the probability falls below the best yet found match. Even with this speed-up mechanism the RegEx-decoder is more than 22 times faster. The running time for the $A^*$-search is growing exponentially as expected but the results match perfectly those of the string-by-string decoding.

The beam search with beam value 100 needs almost seven times more time for the calculation than the RegEx-Decoder. A point of criticism might be that we use no independent implementation to compare the time complexity and we may not implemented the beam search algorithm optimally. Figure \ref{fig:digitAutomaton} shows the corresponding extended automaton. Let us count the multiplications: Beam search with beam width 100 calculates for each of the 100 prefixes at each time step 11 new prefixes (one for each digit plus one adding the \nac{}) and thus 1100 multiplications per time step in total in the worst case. The RegEx-Decoder calculates for each of the 10 arcs which read digits 6 new prefixes and for each of the 6 transitions requiring a \nac{} there is only one multiplication. Thus, we have 66 multiplications in total. Therefore, our theoretical analysis rather indicates that the RegEx-Decoder is implemented suboptimally since beam search needs 16 times more multiplications. Although beam search with beam width 100 is much slower it yields significantly more errors (round about 40 errors on average if the ground truth are 4 or 5 digits). To get a comparable performance for the experiments with 4 and 5 digits, we need a beam width of at least 1000\footnote{The running time increases from round about 15 sec to 115 sec for all 10,000 output matrices.}.
\begin{figure}
\tikzstyle{every state}=[inner sep=3pt,minimum size=10pt]
\begin{tikzpicture}[shorten >=1pt,node distance=1.2cm,on grid,auto]
\node[state,initial] (q0) at (0,0) {};
\node[state,below of=q0] (nac1) {};
\node[state,right=2cm of q0] (q1) {};
\node[state,below of=q1] (nac2) {};
\node[state,right=2cm of q1] (q2) {};
\node[state,below of=q2] (nac3) {};
\node[state,right=2cm of q2,accepting] (q3) {};
\node[state,below of=q3,accepting] (nac4) {};
\node[state,right=2cm of q3,accepting] (q4) {};
\node[state,below of=q4,accepting] (nac5) {};
\node[state,right=2cm of q4,accepting] (q5) {};
\node[state,below of=q5,accepting] (nac6) {};

\path[->] (q0) edge node[left=0mm] {$\symbnac$} (nac1);
\path[->] (q1) edge node[left=0mm] {$\symbnac$} (nac2);
\path[->] (q2) edge node[left=0mm] {$\symbnac$} (nac3);
\path[->] (q3) edge node[left=0mm] {$\symbnac$} (nac4);
\path[->] (q4) edge node[left=0mm] {$\symbnac$} (nac5);
\path[->] (q5) edge node[left=0mm] {$\symbnac$} (nac6);

\path[->] (q0) edge node[above=0mm] {\tt [0-9]} (q1);
\path[->] (q1) edge node[above=0mm] {\tt [0-9]} (q2);
\path[->] (q2) edge node[above=0mm] {\tt [0-9]} (q3);
\path[->] (q3) edge node[above=0mm] {\tt [0-9]} (q4);
\path[->] (q4) edge node[above=0mm] {\tt [0-9]} (q5);

\path[->] (nac1) edge node[pos=0.4,sloped,above=0mm] {\tt [0-9]} (q1);
\path[->] (nac2) edge node[pos=0.4,sloped,above=0mm] {\tt [0-9]} (q2);
\path[->] (nac3) edge node[pos=0.4,sloped,above=0mm] {\tt [0-9]} (q3);
\path[->] (nac4) edge node[pos=0.4,sloped,above=0mm] {\tt [0-9]} (q4);
\path[->] (nac5) edge node[pos=0.4,sloped,above=0mm] {\tt [0-9]} (q5);

\end{tikzpicture}
\caption{Extended automaton accepting 3 to 5 digits and optional intermediate \nac{}s. }\label{fig:digitAutomaton}
\end{figure}
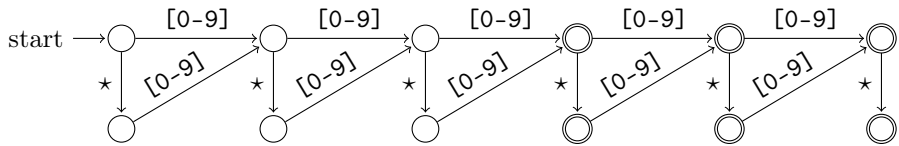

\section{Conclusion}
In this article, we consider regular expressions for the decoding of neural network outputs. Regular expressions are a very efficient way to define a pattern of interest to search in text strings. We suggest to use this pattern for a convenient and clear decoding process. Similar results may also be archived by a smart evaluation of the best path. The advantage of regular expressions over individual evaluation of the output is the simple and unified notation. Furthermore, the proposed algorithm allows a highly adaptable decoding process since only the regular expression has to be changed.

We show how to exploit finite automata to find the most likely feasible label sequence of a regular language.
A further analysis of the decoding procedure yields a speed-up of the algorithm such that it also works fast for complex regular expressions or many network outputs. We propose also an approximation which is shown to be exact under conditions which are commonly satisfied for CTC-trained networks. This theoretical result was confirmed by experiments. As a main result, we showed that the decoder is applicable in practical scenarios. Even if the approximation fails to produce exact results, it is likely that the ground truth does not fit to the regular expression. This results in a low probability decoding result further underestimated by our approximation which should not be harmful in most applications.
Additionally, these experiments show that the proposed method is very efficient compared to state of the art decoding algorithms.

The proposed speed-ups work only for the path probability $\p(\fett \pi|\fett X)$ (instead of the word probability). If the decoder should return the
exact probability, all paths contribute
to the result and, thus, cannot be skipped. Hence, speed-ups seem to be hard.
Additionally, we have to take care about distinct paths through the automaton
accepting the same label sequence. An \emph{Unambiguous FSA} or even a DFA is required to ensure that the automaton accepts every path (of labels) only once. We
already discussed the disadvantages of DFAs in Section \ref{sec:automata}.

There are plenty of applications for the proposed algorithm. The method can be applied e.g. to keyword spotting but also patterns of image retrieval tasks can be described conveniently.
The proposed decoder is an essential part of our handwriting recognition systems e.g. for HTRtS (full text recognition) and ANWRESH (form reading) competitions.

\section{Acknowledgment}
This research was supported by the research grant no. KF2622304SS3 (Zentrales Innovationsprogramm Mittelstand) of the Federal Ministry for Economic Affairs and Energy (Germany).

The authors would like to thank the anonymous reviewers for their valuable comments improving the manuscript. We also thank U. Siewert for his detailed and profound comments.

\section*{Appendix}\label{sec:proofs}
\begin{proof}[Proof of Theorem \ref{theo:exactAlpha}]
For $t=1$, the claim is correct since the most likely path of
length 1 consists of the most likely character if $q' = q_0$. Otherwise we obtain zero.

Let $t>1$.
To keep things simple, we fix $q''$ to consider only prefixes through $(q'',q')$ and $(q',q)$. Therefore, let $\alpha_{t,q'',q',q}^1$ be the likelihood of the most likely prefix through $q'',q'$ and $q$ and let $\emission_{t,q'',q',q}^1$ the read label at
$t$. Then
\[ \app(t+1,q',q,1) = \max_{q''\in\raum P(q')}\alpha_{t+1,q'',q',q}^1.\]
Analogously, let $\alpha_{t,q'',q',q}^2$ be the
most likely feasible sequence through $q'',q'$ and $q$ not ending on
$\emission_{t,q',q}^1$. Then
\[ \app(t+1,q',q,2) = \max_{q''\in\raum P(q') }\alpha_{t+1,q'',q',q}^2.\] The theorem is
proven if
\begin{align*}
\alpha_{t+1,q'',q',q}^1 &= \max\Big\{\alpha^i_{t,q'',q'}y_{t+1,a}\;|\; i\in\{1,2\}, a\in\Gamma(t+1,i,q'',q',q) \Big\} \\
\alpha^2_{t+1,q'',q',q}&=\max\Big\{\alpha^i_{t,q'',q'}y_{t+1,a}\;|\;
i\in\{1,2\},
a\in\Gamma(t+1,i,q'',q',q)\setminus\{\emission^1_{t+1,q',q}\}\Big\}.
\end{align*}
We make a case distinction, calculate the exact probability and show that
$\alpha^i_{t+1,q'',q',q}$ only depends on $\alpha^i_{t+1,q'',q'}$ for
$i\in\{1,2\}$ and $\likeliest_{t+1,q',q}^j$ for $j\in\{1,2,3\}$. For sake of
simplicity, we omit the index $q',q$ for $\likeliest_{t+1,q',q}^j$ for the
rest of the proof 
Analogously, we omit $q'',q'$ for $\alpha^i_{t,q'',q'}$. Thus,
$\alpha^i_{t}=\alpha^i_{t,q'',q'}$ and
$\likeliest_{t+1}^i=\likeliest_{t+1,q',q}^i$. We check the following cases:
\begin{figure}
\centering
\label{fig:alpha1}%
\begin{tikzpicture}[auto,shorten >=3pt]
\node[state] (q'') at (0,0) {$q''$}; \node[state] (q') at (3,0)
{$q'$}; \node[state] (q) at (6,0) {$q$}; \draw[->] (q'') edge node
(e2) [above]{$\emission_{t,q'',q'}^1$} (q'); \draw[->] (q') edge node
(e1) [above]{$\emission_{t+1,q',q}^1$} (q); \node (a1) [below =of e2]
{$\alpha_{t}^1$}; \node (a2) [below =of a1] {$\alpha_{t}^2$}; \node
(l1) [below =of e1] {$y_{t+1,\likeliest_{t+1}^1}$}; \node (l2) [below
=of l1] {$y_{t+1,\likeliest_{t+1}^2}$};
\draw[->] (a1) edge (l1); \draw[->] (a1) edge (l2); \draw[->] (a2)
edge (l1);
\end{tikzpicture}
~ 
\caption{Most likely suffix trough $q'',q'$ and $q$ and possible combinations
to calculate $\alpha^1_{t+1,q'',q',q}$.}
\end{figure}
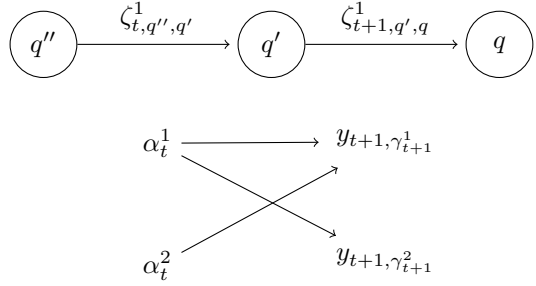

\renewcommand{\theenumi}{\bf \arabic{enumi}}
\renewcommand{\labelenumi}{\theenumi:}
\renewcommand{\theenumii}{\bf \alph{enumii}}
\renewcommand{\labelenumii}{\theenumii:}
\begin{enumerate}
\item \label{case:1} $\emission_{t,q'',q'}^1 \neq \likeliest_{t+1}^1$, i.e. there are no
restrictions by $\mathcal F$. Hence, the most likely path combines the
most likely path through $q'',q'$ with the most likely label
$\likeliest_{t+1}^1$ at arc $(q',q)$:
\begin{align*}
\alpha_{t,q'',q',q}^1&=\alpha_{t}^1y_{t+1,\likeliest_{t+1}^1}
\end{align*}
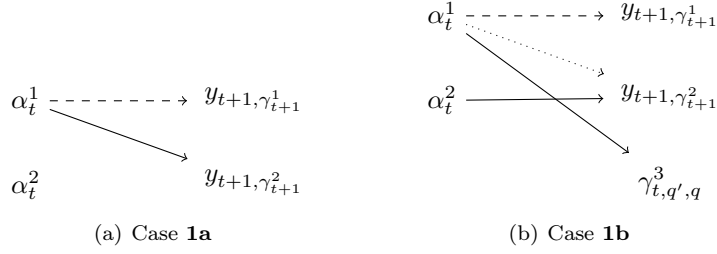
\begin{figure}
{
\hfill
\subfigure[][Case \ref{case:1a}]{\label{fig:1a}
\begin{tikzpicture}[auto,shorten >=3pt,node distance=0.5cm]
\node (a1) at (0,0) {$\alpha_{t}^1$};
\node (a2) [below =of a1] {$\alpha_{t}^2$};
\node (l1) at (3,0) {$y_{t+1,\likeliest_{t+1}^1}$};
\node (l2) [below =of l1] {$y_{t+1,\likeliest_{t+1}^2}$};
\draw[dashed,->] (a1) edge (l1);
\draw[->] (a1) edge (l2);
\end{tikzpicture}}
\hfill
\subfigure[][Case \ref{case:1b}]{\label{fig:1b}
\begin{tikzpicture}[auto,shorten >=3pt,node distance=0.5cm]
\node (a1) at (0,0){$\alpha_{t}^1$};
\node (a2) [below =of a1] {$\alpha_{t}^2$};
\node (l1) at (3,0) {$y_{t+1,\likeliest_{t+1}^1}$};
\node (l2) [below =of l1] {$y_{t+1,\likeliest_{t+1}^2}$};
\node (l3) [below =of l2] {$\likeliest_{t,q',q}^3$};
\draw[dashed,->] (a1) edge (l1);
\draw[dotted,->] (a1) edge (l2);
\draw[->] (a2) edge (l2);
\draw[->] (a1) edge (l3);
\end{tikzpicture}}
\hfill
}
\caption{Subcases of \ref{case:1}. Combination of $\alpha^1_{t+1,q'',q',q}$ dashed, other forbidden paths are dotted. Solid arcs denote possible combinations to calculate $\alpha_{t+1,q'',q',q}^2$.}
\end{figure}
\begin{enumerate}
\item \label{case:1a}
$\emission_{t,q'',q'}^1 \neq \likeliest_{t+1}^2$ (see Figure \ref{fig:1a}).
There are no restrictions such that the second most likely path is 
\begin{align*}
\alpha_{t+1,q'',q',q}^2&=\alpha_{t}^1y_{t+1,\likeliest_{t+1}^2}.
\end{align*}
\item \label{case:1b}
$\emission_{t,q'',q'}^1 = \likeliest_{t+1}^2$ (dotted combination in Figure \ref{fig:1b}). The suffices $\likeliest_{t+1}^1$ and $(\emission_{t,q'',q'}^1,\likeliest_{t+1}^2)$ are not allowed
in
$\Pi(t+1,q',q)$. Thus,
\begin{align*}
\alpha_{t+1,q'',q',q}^2&=\max\left\{\alpha_{t}^1y_{t+1,\likeliest_{t+1}^3}\:
,\:\alpha_{t}^2y_{t+1,\likeliest_{t+1}^2}\right\}.
\end{align*}
\end{enumerate}

\item \label{case:2}
$\emission_{t,q'',q'}^1 =\likeliest_{t+1}^1$. Thus, due to $\mathcal F$, it is not allowed that consecutive arcs read the same label at consecutive positions. The most likely path from $\Pi(t+1,q',q)$ through $q'', q'$ and $q$ combines either the most likely path from $\Pi(t,q'',q')$ with
the second most likely label at position $t+1$ or the second most
likely path from $\Pi(t,q'',q')$ with
the most likely label at position $t+1$:
\begin{align*}
\alpha_{t,q'',q',q}^1&=\max\left\{\alpha_{t}^1y_{t+1,\likeliest_{t+1}^2}\:
,\:\alpha_{t}^2y_{t+1,\likeliest_{t+1}^1}\right\}
\end{align*}

\begin{enumerate}
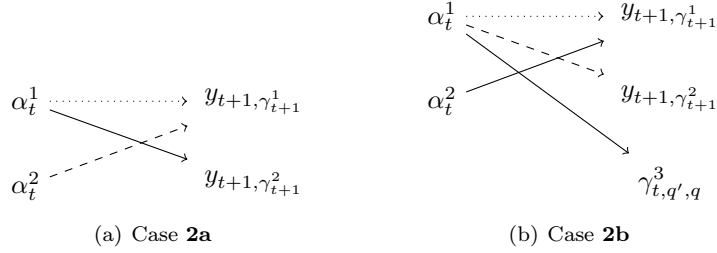
\begin{figure}
{
\hfill
\subfigure[][Case \ref{case:2a}]{
\label{fig:2a}
\begin{tikzpicture}[auto,shorten >=3pt,node distance=0.5cm]
\node (a1) at (0,0) {$\alpha_{t}^1$};
\node (a2) [below =of a1] {$\alpha_{t}^2$};
\node (l1) at (3,0) {$y_{t+1,\likeliest_{t+1}^1}$};
\node (l2) [below =of l1] {$y_{t+1,\likeliest_{t+1}^2}$};
\draw[dotted,->] (a1) edge (l1);
\draw[dashed,->] (a2) edge (l1);
\draw[->] (a1) edge (l2);
\end{tikzpicture}}
\hfill
\subfigure[][Case \ref{case:2b}]{
\label{fig:2b}
\begin{tikzpicture}[auto,shorten >=3pt,node distance=0.5cm]
\node (a1) at (0,0){$\alpha_{t}^1$};
\node (a2) [below =of a1] {$\alpha_{t}^2$};
\node (l1) at (3,0) {$y_{t+1,\likeliest_{t+1}^1}$};
\node (l2) [below =of l1] {$y_{t+1,\likeliest_{t+1}^2}$};
\node (l3) [below =of l2] {$\likeliest_{t,q',q}^3$};
\draw[dashed,->] (a1) edge (l2);
\draw[dotted,->] (a1) edge (l1);
\draw[->] (a2) edge (l1);
\draw[->] (a1) edge (l3);
\end{tikzpicture}}
\hfill
}
\caption{Subcases of \ref{case:2}. Combination of $\alpha^1_{t+1,q'',q',q}$ dashed, other forbidden paths are dotted. Solid arcs denote possible combinations to calculate $\alpha_{t+1,q'',q',q}^2$.}
\end{figure}
\item $\emission_{t+1,q'',q',q}^1 = \likeliest_{t+1}^1 $ (dashed combination in Figure \ref{fig:2a}).\label{case:2a}
The only restriction is not to read $\likeliest_{t+1}^1$ such that the second most likely path is simply 
\begin{align*}
\alpha_{t+1,q'',q',q}^2&=\alpha_{t}^1y_{t+1,\likeliest_{t+1}^2}.
\end{align*}
\item \label{case:2b} $\emission_{t+1,q'',q',q}^1 = \likeliest_{t+1}^2 $ (dashed combination in Figure \ref{fig:2b}). Hence, the
suffices $(\emission_{t,q'',q'}^1,\likeliest_{t+1}^1)$ and
$\likeliest_{t+1}^2$ are forbidden.
\begin{align*}
\alpha_{t+1,q'',q',q}^2&=\max\left\{\alpha_{t}^2y_{t+1,\likeliest_{t+1}^1}\:
,\:\alpha_{t}^1y_{t+1,\likeliest_{t+1}^3}\right\}
\end{align*}
\end{enumerate}
\end{enumerate}

This completes the proof.

\end{proof}

\begin{proof}[Proof of Theorem \ref{theo:approxAlpha}]
Note, that the approximation is exact for arcs reading two or less labels since in this case $\widetilde\cont(t,q',q,i)$ and $\cont(t,q',q,i)$ maximize over the same paths. Especially, \nac{}-transitions are always exact since they read only one label.

Let $\fett \pi^*$ be the most likely feasible path with respect to the regular expression $\fett r$ and let
\begin{align*}
\cont(t,q',q,i) = \prod_{t'=1}^ty_{t',\pi^*_{t'}}
\end{align*}
for $i\in\{1,2\}$, i.e. $\pi^*_t=\pi^*_{t-1}$.
Assume $\pi^*_t\in \albet$.
Due to Assumption \ref{cond:limitedContinues} $\pi^*_{t-2},\pi^*_{t+1}\neq\pi^*_{t}$. 

The likelihood of $(\pi^*_1,\dots,\pi^*_{t-1})$ is equal to $\alpha_{t-1,q',q}^j$ for some $j$. $j$ cannot be greater than 2 since otherwise $\pi^*_{t-1}\not \in \{\likeliest_{t-1,q',q}^{1},\likeliest_{t-1,q',q}^{2}\}$ (see Figure \ref{fig:1a}) and the substitution of $\pi^*_{t-1}$ by $\symbnac$ would yield a feasible path with greater likelihood due to condition \ref{cond:likelynac}. This contradicts to the assumption that $\fett \pi^*$ is maximizing the likelihood of all feasible paths. Thus, we only need to compute $\alpha^i_{t,q',q}$ for $i\leq2$.

If $\pi^*_{t}\not \in \{\likeliest_{t,q',q}^{1},\likeliest_{t,q',q}^{2},\likeliest_{t,q',q}^{3}\}$, we get a feasible, more likely path by substituting $\pi^*_{t}$ by $\symbnac$. This new path collapses to the same word. Again, this is a contradiction to the maximum likelihood of $\fett \pi^*$. Thus, we only need to consider the three most likely labels per arc.

\end{proof}

\bibliography{lit}

\end{document}